\documentclass[11pt]{article}

\usepackage[utf8]{inputenc}
\usepackage[T1]{fontenc}
\usepackage[margin=1in]{geometry}
\usepackage{hyperref}
\usepackage{url}
\usepackage{booktabs}
\usepackage{amsfonts}
\usepackage{nicefrac}
\usepackage{microtype}
\usepackage[table]{xcolor}

\usepackage{graphicx}
\usepackage{amsthm}

\newtheorem{theorem}{Theorem}
\usepackage{amsmath}
\usepackage{cleveref}

\usepackage{bm}
\usepackage{algorithm}
\usepackage{algorithmic}

\usepackage{enumerate}
\usepackage{enumitem}
\usepackage{multirow}
\usepackage{multicol}
\usepackage{makecell}
\usepackage{subcaption}
\usepackage{float}
\usepackage{siunitx}

\newcommand{\best}[1]{\textbf{#1}}

\DeclareMathOperator*{\argmax}{arg\,max}
\DeclareMathOperator*{\argmin}{arg\,min}

\def\Acal{{\mathcal{A}}}
\def\Bcal{{\mathcal{B}}}
\def\Fcal{{\mathcal{F}}}
\def\Ical{{\mathcal{I}}}
\def\Lcal{{\mathcal{L}}}
\def\Ncal{{\mathcal{N}}}
\def\Ocal{{\mathcal{O}}}
\def\Pcal{{\mathcal{P}}}
\def\Scal{{\mathcal{S}}}

\def\Ibb{{\mathbb{I}}}
\def\Rbb{{\mathbb{R}}}

\def\xbm{{\bm{x}}}
\def\ybm{{\bm{y}}}
\def\mubm{{\bm{\mu}}}

\def\Xbf{{\mathbf{X}}}

\title{DICS: Data-Informed Centroid Splitting for Decision Tree Classifiers}

\author{
  MD Saifur Rahman Mazumder\thanks{\texttt{msmazumder@miners.utep.edu}}
  \quad Feng Yu\thanks{Corresponding author. \texttt{fyu@utep.edu}} \\[4pt]
  Department of Mathematical Sciences\\
  University of Texas at El Paso
}
\date{}

\begin{document}

\maketitle

\begin{abstract}
Decision tree–based models are widely used in machine learning due to their interpretability and strong empirical performance. However, training decision trees can be computationally expensive, particularly for large and high-dimensional datasets, largely due to the exhaustive search over candidate splits at each node. To improve computational efficiency, we propose \textit{Data-Informed Centroid Splitting (DICS)}, a clustering-based framework that constructs a compact and informative set of candidate splits using data-driven priors. By incorporating class-aware structure, DICS significantly reduces the split search space for \textit{classification tasks} while preserving predictive performance. We further provide theoretical analysis showing that under the stated assumptions, DICS does not degrade the performance of classification trees compared to exhaustive split search.
DICS can be incorporated into
classification trees, random forests, and gradient-boosting models.
Extensive experiments demonstrate that DICS achieves comparable accuracy while substantially reducing training time across synthetic and benchmark datasets, highlighting the benefit of integrating data-informed priors into split selection for scalable classification tree learning.
\end{abstract}

\section{Introduction}

Data-driven methods have become fundamental across a wide spectrum of real-world applications, where systems must extract insights from large and complex datasets to make reliable decisions. Tasks such as spam filtering, targeted advertising, fraud detection, and scientific discovery all depend on models that can capture intricate patterns while remaining computationally efficient. Among the broad family of machine learning techniques, decision tree--based models are particularly appealing due to their interpretability, flexibility, and strong empirical performance. Ensemble variants, including random forests~\cite{breiman2001random} and gradient boosting machines~\cite{friedman2001greedy}, are especially effective, as they can model nonlinear relationships and higher-order feature interactions while maintaining relatively low computational overhead. Their hierarchical structure further enables them to handle heterogeneous data and provide interpretable decision rules, making them a key component of modern scalable learning systems.

\textit{Decision tree} (DT)~\cite{breiman2017classification}, also known as \emph{Classification and Regression Tree (CART)}, is a simple and interpretable model that produces predictions by recursively partitioning the feature space through a sequence of binary decisions. Despite its conceptual simplicity, training a decision tree can be computationally demanding. At each node, the algorithm evaluates a large number of candidate split points, which are typically derived from sorted feature values, and selects the optimal split according to a predefined criterion. This \textit{exhaustive search process} is repeated independently across nodes, leading to significant computational costs as the number of samples, features, or tree depth increases.

To address these challenges, a substantial body of work has focused on improving the efficiency and scalability of tree construction. Early approaches such as SLIQ~\cite{mehta1996sliq} and RainForest~\cite{gehrke1998rainforest} employ presorting techniques, breadth-first growth strategies, and compact data representations to decouple scalability from split evaluation. QUEST~\cite{loh1997split} addresses a related problem by enabling fast tree construction while reducing variable-selection bias. \cite{chickering2001efficient} proposed dynamic split-point selection strategies that restrict the number of candidate thresholds for continuous features; however, these methods often rely on restrictive assumptions and may not scale well in high-dimensional settings. In the context of ensemble learning, Extremely Randomized Trees~\cite{geurts2006extratrees} further reduce computational cost by introducing randomness in both feature selection and split thresholds, thereby avoiding exhaustive search. More recent systems, such as XGBoost~\cite{chen2016xgboost} and LightGBM~\cite{ke2017lightgbm}, achieve substantial efficiency gains through approximate split finding, sparsity-aware computation, and histogram-based binning techniques. Collectively, these developments highlight that split selection remains a primary computational bottleneck in tree-based learning.

In this paper, we address the computational challenges of \textit{Classification
Trees} by exploiting structural information in the data to guide split
selection.
In classification tasks, the underlying data patterns provide valuable structural information that can guide split selection. 
Motivated by this observation, we propose \textit{\textbf{Data-Informed Centroid Splitting (DICS)}}, a method that leverages data-driven priors through clustering to construct a significantly reduced set of candidate splits, thereby substantially accelerating the tree training process.

Beyond conventional greedy tree induction, another line of work considers
globally optimized sparse decision trees, where branch-and-bound and related
search strategies optimize the complete tree structure rather than selecting
splits greedily~\cite{hu2019osdt,zhang2023optimal,lin2020generalized,aglin2020learning}.
Within this setting, McTavish et al.~\cite{mctavish2022fast}, for example,
use a boosted tree ensemble as a reference model to identify informative
thresholds for accelerating the optimization of sparse decision trees.
While DICS shares the general motivation of restricting the continuous split
space, it derives the candidates directly from the clustering structure of
the input data without requiring an auxiliary predictive model. Clustering
has also been incorporated more directly into tree construction,
Clus-DTI~\cite{barros2011clustering} uses clustering to decompose a classification problem
into simpler subproblems before inducing decision trees, while
BDTKS~\cite{wang2020linear} recursively applies K-means within a multivariate binary tree and converts centroid-based separations into hyperplane tests.
In contrast, DICS uses clustering only to construct a reusable set of
univariate feature--threshold candidates, leaving the underlying greedy
tree objective unchanged and allowing the same candidate-generation
strategy to be used with classification trees, random forests, and
gradient-boosting models.

In summary, our contributions are as follows:
\begin{enumerate}[nosep]
    \item We propose \textit{Data-Informed Centroid Splitting (DICS)}, a clustering-driven approach for constructing compact candidate split sets in classification trees.
    \item We provide theoretical analysis showing that decision trees built with DICS preserve predictive performance compared to standard exhaustive split search.
    \item We conduct extensive experiments demonstrating that DICS achieves comparable accuracy while significantly improving computational efficiency.
\end{enumerate}

\section{Preliminaries}

\subsection{Classification trees}\label{sec:prelim_dt}

We consider a multi-class classification problem. Let $\Xbf=[\xbm_1,\ldots,\xbm_N]\in\mathbb{R}^{P\times N}$ denote the feature matrix and $\ybm\in[K]^N$ the corresponding labels, where $N$ is the number of samples, $P$ is the number of features, and $K$ is the number of classes. Formally, a classification tree can be defined by $T(\xbm;\bm{\theta})=\sum_{m=1}^Mc_m\Ibb{(\xbm\in R_m)}$ where $\bm{\theta}=\{(c_m,R_m)\}_{m=1}^M$ is the parameter, $R_m$ denotes the $m$-th rectangle region, $c_m \in [K]$ is the predicted class assigned to $R_m$, and $M$ is the total number of regions. To fit the model, we minimize the following empirical risk:
\begin{align}\label{eq:obj_1}
    \Lcal(\bm{\theta}) = \sum_{i=1}^N\ell(y_i,T(\xbm_i;\bm{\theta}))=\sum_{m=1}^M\sum_{\xbm_i\in R_m}\ell(y_i,c_m).
\end{align}
Solving \eqref{eq:obj_1} typically involves two stages: (1) learning the discrete tree structure $\{R_m\}_{m=1}^{M}$, and (2) estimating the optimal labels $c_m$ given the partition $\{R_m\}_{m=1}^{M}$. However, the first step makes \eqref{eq:obj_1} non-differentiable and computationally challenging. In fact, finding the optimal data partition induced by a decision tree is NP-complete~\cite{laurent1976constructing}.

Each region $R_m$ corresponds to a leaf node and is formed by a sequence of recursive binary splits. Equivalently, every $R_m$ can be expressed as the intersection of splitting constraints along the path from the root to the corresponding leaf: $R_m = \bigcap_{t \in \mathcal{P}_m} \left\{\xbm \in \mathbb{R}^P : x_{f_t} \in S_t \right\}$,
where $\mathcal{P}_m$ denotes the set of internal nodes along the path to leaf $m$, $f_t$ is the feature selected at node $t$, and $S_t \in \{(-\infty,\tau_t), [\tau_t,\infty)\}$ specifies the split induced by threshold $\tau_t$.
A feature-threshold pair $(f,\tau)$ defines a split that partitions the index set $[N]=\{1,\ldots,N\}$ into two subsets:
\begin{align*}
\mathcal{I}_L &= \{i\in\mathcal{I} : \Xbf_{f,i} < \tau\},\quad 
\mathcal{I}_R = \{i\in\mathcal{I} : \Xbf_{f,i} \ge \tau\}.
\end{align*}

The search for the optimal splits is typically performed using a greedy, top-down strategy. Starting from the root node, all candidate splits are evaluated at each node. Specifically, for each feature $f$, the samples are first sorted according to their values along that feature, and candidate thresholds are chosen as the midpoints between consecutive distinct values.
The quality of a split $(f,\tau)$ is measured by the weighted post-split impurity:
\begin{align}\label{eq:quality_measure}
Q(f,\tau) = \frac{|\mathcal{I}_L|}{N} G(\mathcal{I}_L) + \frac{|\mathcal{I}_R|}{N} G(\mathcal{I}_R).
\end{align}
Here, $G(\mathcal{I})$ denotes the impurity of a node indexed by $\mathcal{I}$. Common choices include the Gini index and entropy (deviance). To compute $G(\mathcal{I})$, we first estimate the empirical class distribution within $\mathcal{I}$ as $\hat{\pi}_c = \frac{1}{|\Ical|} \sum_{i\in\Ical} \Ibb(y_i = c)$. The Gini index is defined as $G(\Ical) = \sum_{c=1}^K \hat{\pi}_c (1 - \hat{\pi}_c) = 1 - \sum_{c=1}^K \hat{\pi}_c^2$~\cite{gini1912variabilita,breiman2017classification}, while the entropy is given by $G(\Ical) = -\sum_{c=1}^K \hat{\pi}_c \log \hat{\pi}_c$~\cite{Entropy}. This impurity-based criterion underlies the standard greedy split selection in classification trees and is widely used in tree-based methods~\cite{breiman2017classification}.

\subsection{Ensemble trees}

Although a single decision tree offers strong interpretability, it typically suffers from high variance, instability, and sensitivity to noise in the training data. In practice, ensemble tree methods are therefore often preferred. These approaches combine multiple decision trees to construct a stronger predictive model. The main types of ensemble trees include \textit{bagging} (or \textit{bootstrap aggregating}), \textit{Random Forest} (RF, a specific instance of bagging)~\cite{breiman2001random}, \textit{boosting}, and etc. 

For classification tasks, an ensemble tree typically makes predictions via majority voting: $\hat{y}(\xbm)=\argmax_{c\in[K]}\sum_{b=1}^B\Ibb(T_b(\xbm)=c)$, where $\{T_b\}_{b=1}^B$ are the individual trees. In bagging, each tree $T_b$ is trained on a bootstrap sample $(\Xbf^{*b}, \ybm^{*b})$. Random Forest further reduces variance by \textit{de-correlating} trees, typically by restricting each tree to a random subset of features (of size $\lfloor \sqrt{P} \rfloor$ for classification). In boosting, trees are constructed sequentially, with each new tree designed to correct the errors made by the previous ones. Among these approaches, gradient boosting methods such as XGBoost~\cite{chen2016xgboost} and LightGBM~\cite{ke2017lightgbm} are widely recognized for achieving strong predictive performance on tabular data.

\subsection{Computational bottleneck}\label{sec:bottleneck}

As discussed in earlier sections, the most challenging and computationally demanding component of a Decision Tree (DT) is learning the discrete tree structure, which is equivalent to searching for optimal splits (i.e., feature–threshold pairs). The standard approach adopts a greedy, top-down strategy: at each node, it \textbf{\textit{exhaustively}} evaluates every split out of $(N-1)P$ candidates, corresponding to $P$ features and up to $(N-1)$ candidate thresholds per feature (typically taken as midpoints between consecutive sorted sample values). As a result, the overall training cost scales rapidly with the sample size, tree depth, and ensemble size.

This \textbf{\textit{large candidate space}} is the main computational bottleneck of DT training. To alleviate this issue, a popular alternative is the histogram-based method~\cite{chen2016xgboost,li2007mcrank,ke2017lightgbm}. This approach discretizes each feature into a finite number of bins and constructs feature histograms, allowing the optimal split to be found by scanning histogram bins instead of raw data values.
However, this binning strategy introduces additional approximation errors. Specifically, multiple distinct feature values are merged into the same bin, and candidate splits are restricted to bin boundaries only. This leads to two main drawbacks: (1) \textit{quantization error} due to coarse discretization of continuous features, and (2) a \textit{trade-off} between efficiency and accuracy that depends on the number of bins used. In particular, using fewer bins improves computational efficiency but increases information loss, while using more bins reduces quantization error at the cost of higher computational and memory overhead.

Nevertheless, both exhaustive search and histogram-based methods lack prior guidance on which splits are likely to be informative. This motivates our approach, which (i) constructs a compact, data-informed set of candidate splits, and (ii) enforces a fixed per-node evaluation budget, ensuring that the split-search cost remains predictable and controlled.

\section{Methodology}\label{sec:methodology}

In this section, we introduce a novel approach, termed \textit{Data-Informed Centroid Splitting (DICS)}, which precomputes a compact, data-informed set of candidate thresholds for use in classification trees (see \Cref{sec:CBS}). This strategy directly addresses the split-search bottleneck without modifying the underlying tree objective.
We then describe how DICS can be integrated into standard classification tree-based frameworks, including classical decision trees, random forests~\cite{breiman2001random}, and gradient boosting systems, in \Cref{sec:CG_methods}. We further analyze the computational complexity of DICS in \Cref{sec:complexity}.

\subsection{Data-Informed Centroid Splitting}\label{sec:CBS}

A fundamental assumption in classification is that samples belonging to the same class are generally close to each other in the feature space. This principle underlies many methods, such as k-nearest neighbors and linear discriminant analysis~\cite{cover1967nearest,duda2006pattern,bishop2006pattern}. Consequently, in ideal scenarios, one might expect the classification boundaries to roughly align with clusters formed by the data even without in the absence of label information.
Although this alignment may not hold exactly in practice, clustering results can still provide valuable guidance and serve as useful prior information. A simple illustrative example is shown in the left panel of \Cref{fig:illustration}: we generate $N=600$ two-dimensional samples for two classes with $\xbm | y=0 \sim \Ncal((-1,0), \Sigma)$ and $\xbm| y=1 \sim \Ncal((1,1), \Sigma)$ and $\Sigma=[0.25,0.1;0.1,0.25]$. In this setting, the optimal Bayes decision boundary (labeled as classification boundary) is well approximated by the separator obtained from applying K-means clustering to the data. 

\begin{figure}[http]
\centering
\newlength{\figheight}
\setlength{\figheight}{4cm}
\begin{minipage}{0.3\textwidth}
    \centering
    \includegraphics[height=\figheight]{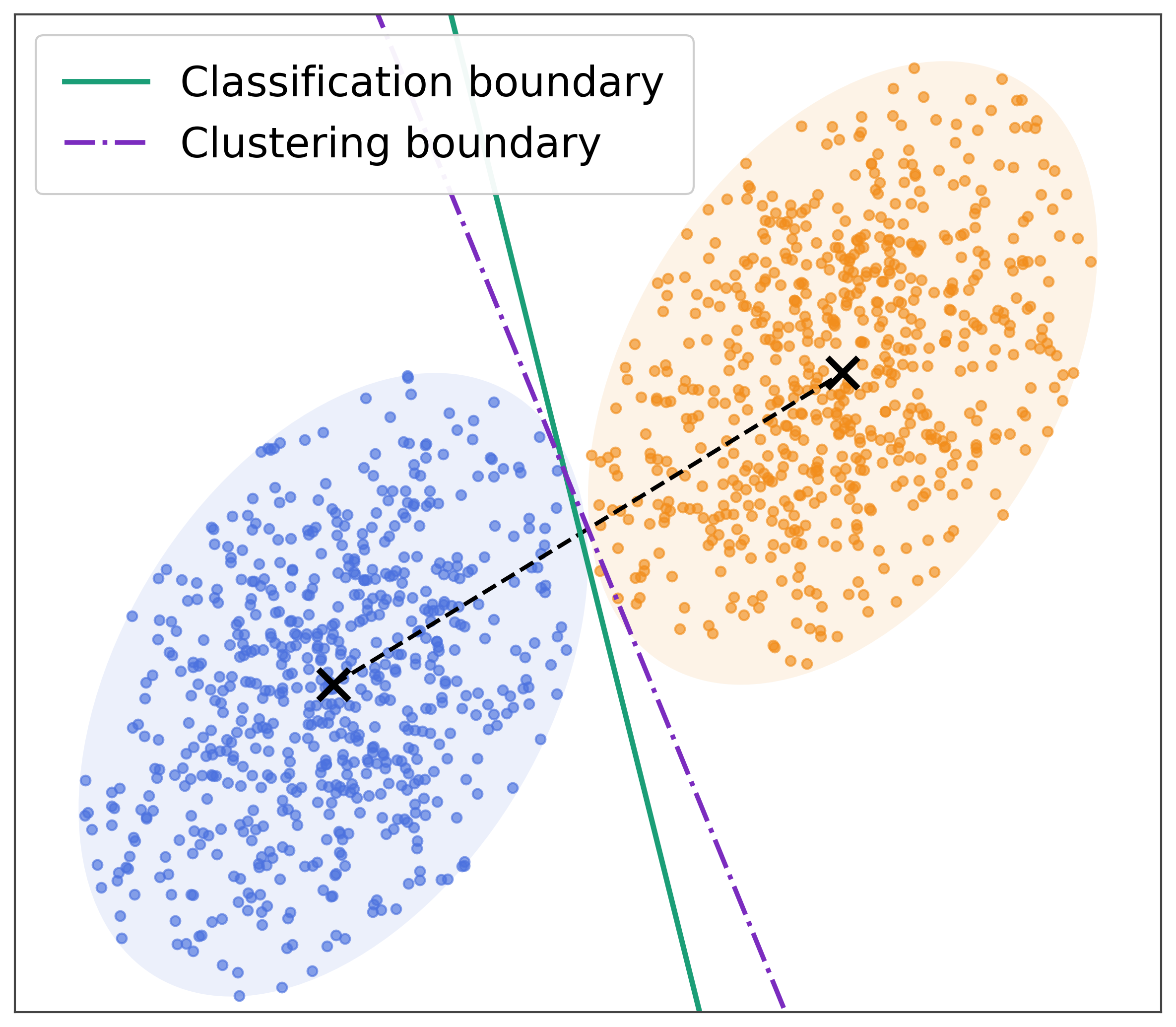}
\end{minipage}
\hspace{0.1\textwidth}
\begin{minipage}{0.3\textwidth}
    \centering
    \includegraphics[height=\figheight]{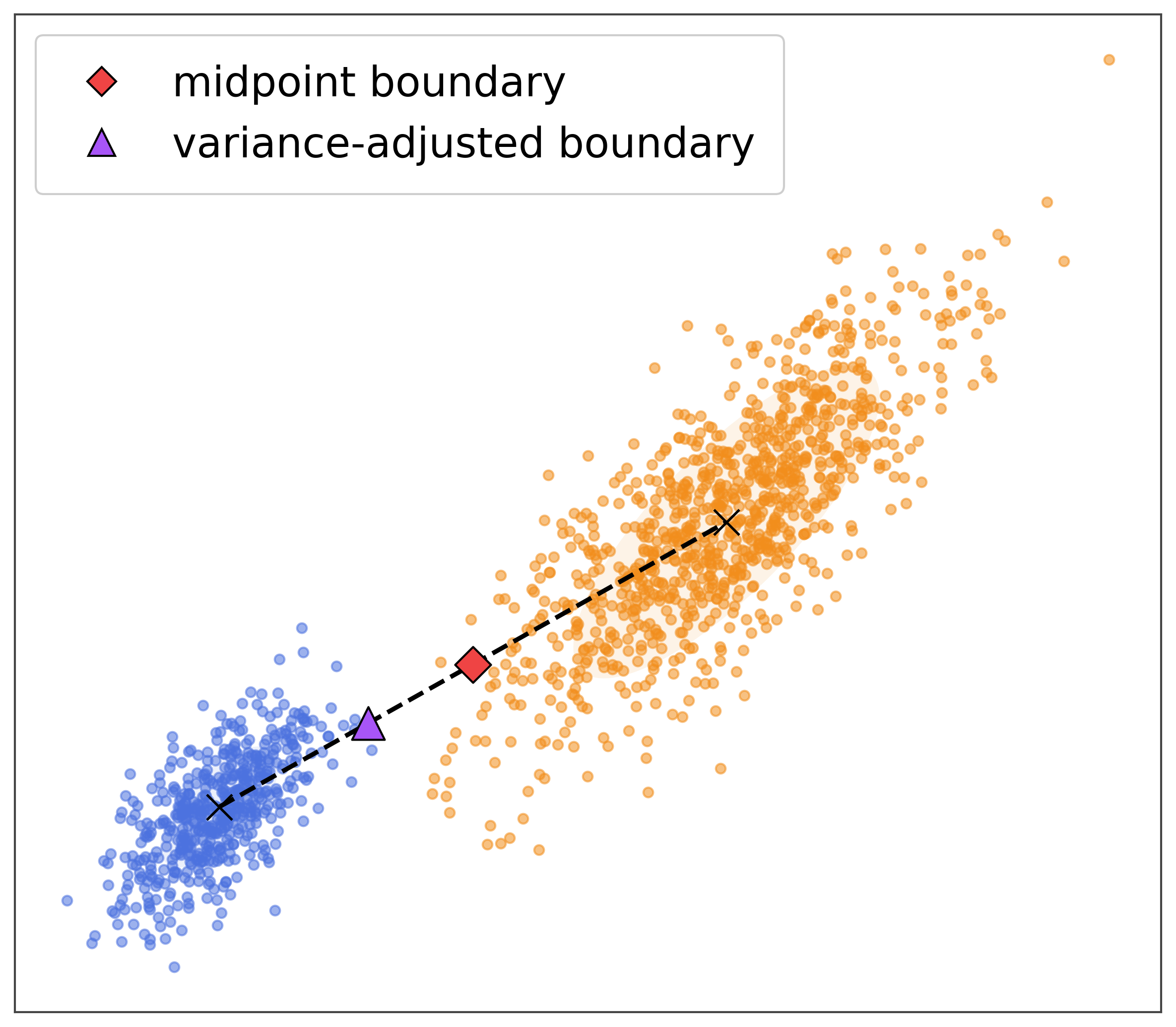}
\end{minipage}
\hspace{-0.05\textwidth}
\begin{minipage}{0.32\textwidth}
    \centering
    \vspace{0.05in}
    \includegraphics[height=0.51\figheight]{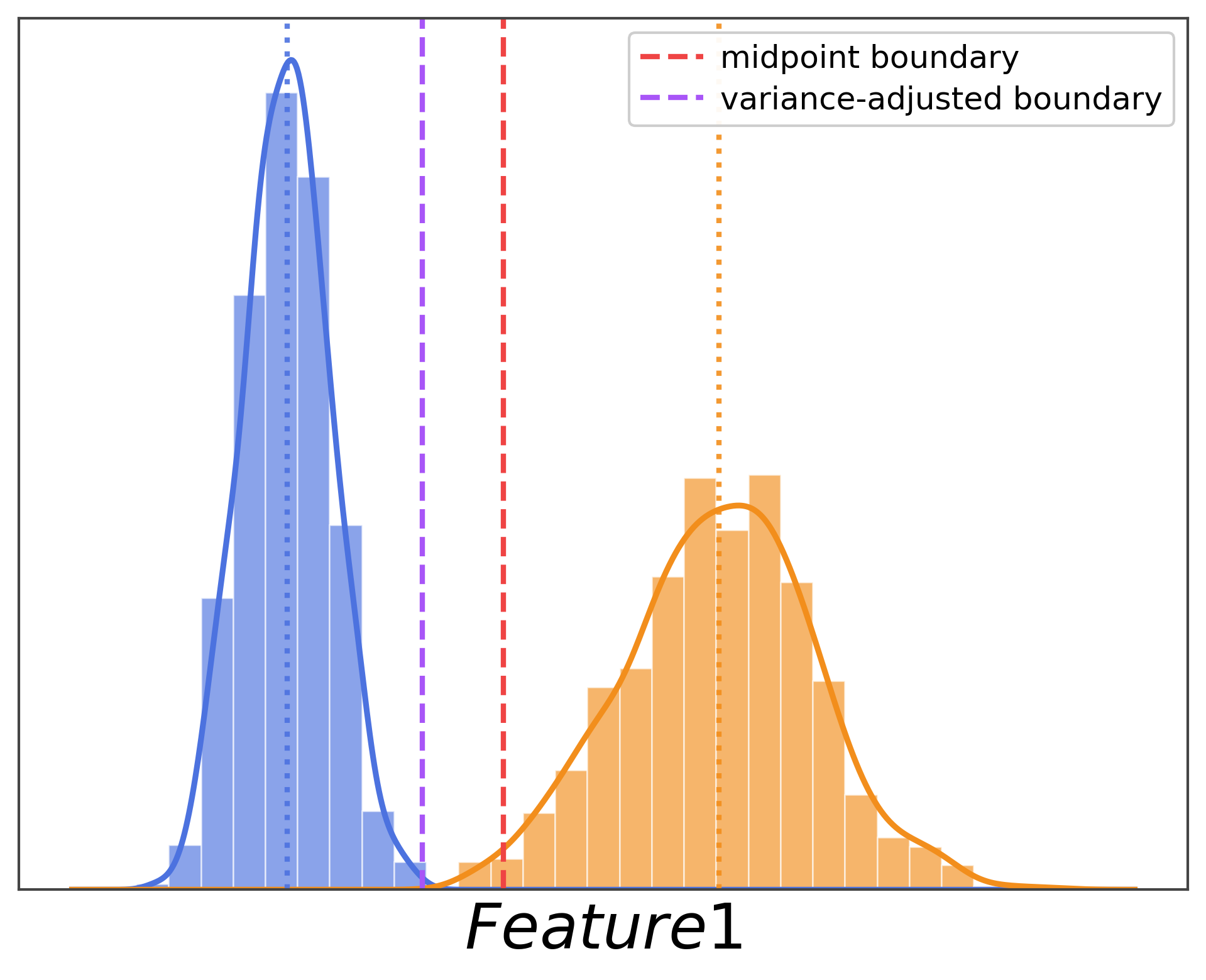}\\
    \includegraphics[height=0.51\figheight]{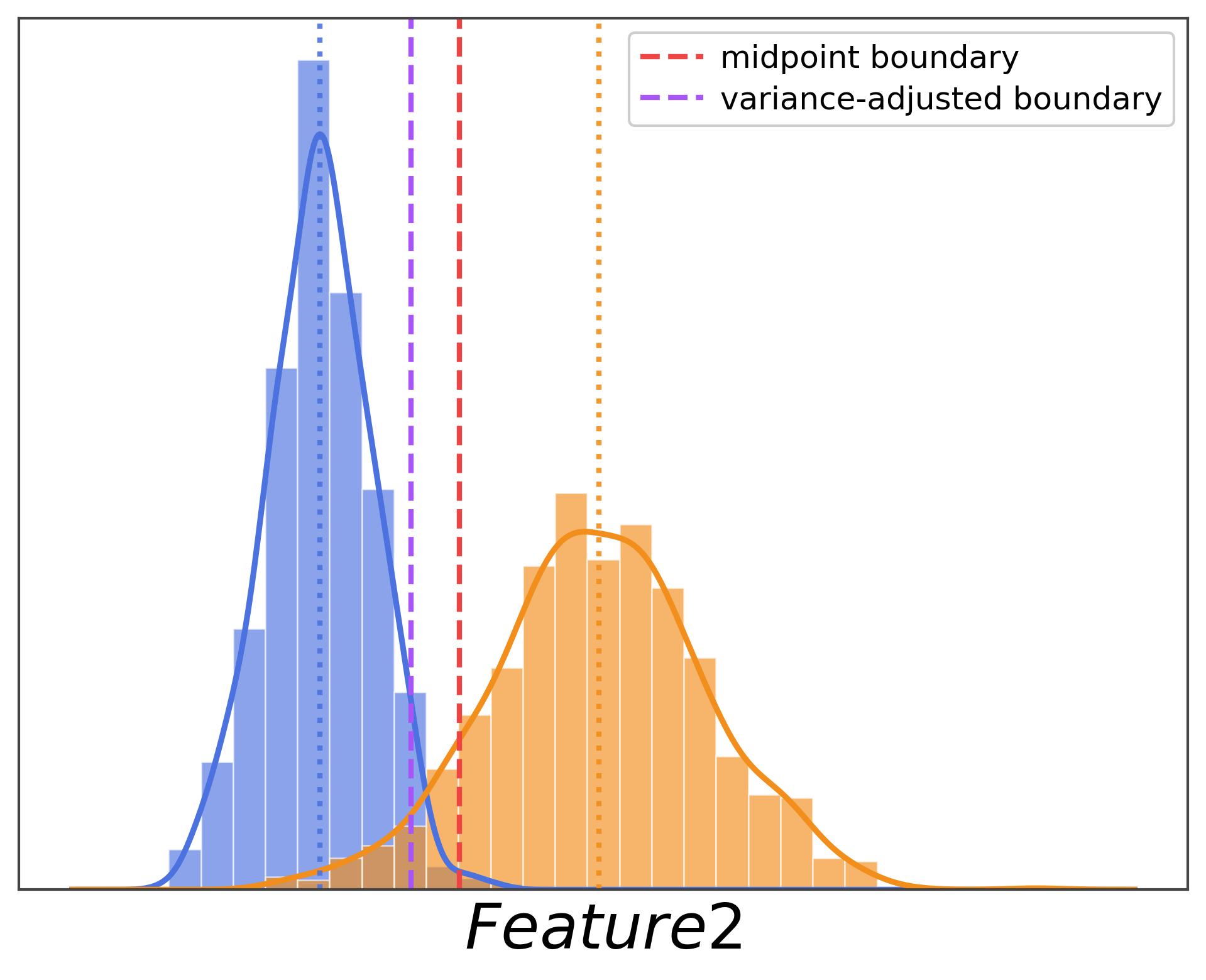}
\end{minipage}
\caption{Illustrative example: (left) classification vs.\ clustering boundaries; (right) midpoint vs.\ variance-adjusted boundaries.}
\label{fig:illustration}
\end{figure}

Moreover, clustering can typically be performed efficiently, for instance, via K-means~\cite{macqueen1967multivariate}. Motivated by the observed alignment between clustering boundaries and classification decision boundaries, as well as the computational effectiveness of K-means, we propose a novel clustering-guided approach for generating candidate splits.
Specifically, let $\{\mubm_c\}_{c=1}^K\subset\Rbb^{P}$ be the centroids obtained from applying K-means clustering to the feature matrix $\Xbf$. 
Each centroid $\mubm_c$ represents a dominant mode of the underlying data distribution, providing a coarse yet stable summary of where the data concentrates. 

We propose \textbf{Data-Informed Centroid Splitting (DICS)} approach, which leverages clustering boundaries as candidate classification boundaries, further yielding splits for tree-based models. In particular, for any pair of clusters $c_1$ and $c_2$, a boundary can be defined based on their centroids. For a given feature $f$, this induces a candidate split of the form $\zeta(\mu_{c_1}^{(f)}, \mu_{c_2}^{(f)})$.
However, in the absence of labels, there is no unique ``correct'' boundary in clustering. In K-means, a point $\xbm$ is assigned to the cluster with the nearest centroid, i.e., $\argmin_{1 \leq c \leq K} \|\xbm - \mubm_c\|$. The boundary between two clusters is thus given by the set of points satisfying $\|\xbm - \mubm_{c_1}\| = \|\xbm - \mubm_{c_2}\|$, which induces a Voronoi partition of the space~\cite{voronoi1908nouvelles}. Under this construction, the corresponding one-dimensional split is given by the midpoint, $\zeta(\mu_{c_1}^{(f)},\mu_{c_2}^{(f)})=\frac{1}{2}(\mu^{(f)}_{c_1}+\mu^{(f)}_{c_2})$. 
Nevertheless, this midpoint-based boundary may be suboptimal when clusters exhibit different variances. For example, if one class is significantly more dispersed than the other (as illustrated in the right panel of \Cref{fig:illustration}), the midpoint can be biased toward the wider cluster. To address this issue, we propose a variance-adjusted boundary in which the distances from the centroids to the boundary are proportional to their standard deviations, i.e. $\frac{\zeta-\mu_{c_1}^{(f)}}{\mu_{c_2}^{(f)}-\zeta}=\frac{\hat{\sigma}_{c_1}}{\hat{\sigma}_{c_2}}$, which implies that 
$\zeta(\mu_{c_1}^{(f)},\mu_{c_2}^{(f)}) = \frac{\hat\sigma_{c_1}\mu^{(f)}_{c_2}+\hat\sigma_{c_2}\mu^{(f)}_{c_1}}{\hat\sigma_{c_1}+\hat\sigma_{c_2}}$, 
where $\hat{\sigma}_{c_1}$ and $\hat{\sigma}_{c_2}$ are the standard deviations of feature $f$ within clusters $c_1$ and $c_2$, respectively. This formulation provides a more adaptive and robust candidate split when cluster spreads differ. 

We also note that there are $\binom{K}{2} = \frac{K(K-1)}{2}$ possible pairs of centroids, whereas $K$ clusters are present. When two centroids are very close, the boundary between them is unlikely to correspond to a meaningful split in the tree. To eliminate such redundant splits, we retain only $m$ centroid pairs corresponding to the $m$ largest inter-centroid distances $D_{ij}=\|\mubm_{i}-\mubm_j\|$. The resulting procedures are summarized in \Cref{alg:split_candidates}.

\begin{algorithm}[h]
\caption{Data-Informed Centroid Splitting (DICS)}
\label{alg:split_candidates}
\begin{algorithmic}[1]
\REQUIRE Feature matrix $\Xbf \in \mathbb{R}^{P \times N}$, maximum pair count $m < \frac{K(K-1)}{2}$.
\ENSURE Split-candidate dictionary $\mathcal{S}$.

\STATE Fit K-means on $\Xbf$ with $K$ clusters to obtain centroids $\{\mubm_c\}_{c=1}^K$

\STATE Compute pairwise centroid distances: $D_{ij}\leftarrow \|\mubm_i - \mubm_j\|, \forall i < j$
\STATE Select $\mathcal{P} \leftarrow$ the $m$ pairs $(i,j)$ with largest $D_{ij}$

\FOR{$f = 1$ \TO $P$}
    \STATE $T_f \leftarrow 
        \left\{ \big(f,\zeta(\mu_i^{(f)},\mu_j^{(f)})\big) : (i,j) \in \mathcal{P} \right\}$, where $\zeta=(\hat\sigma_{c_1}\mu^{(f)}_{c_2}+\hat\sigma_{c_2}\mu^{(f)}_{c_1})/(\hat\sigma_{c_1}+\hat\sigma_{c_2})$.
\ENDFOR
\STATE \textbf{return} $\mathcal{S}=\cup_{f=1}^P\{T_f\}$.

\end{algorithmic}
\end{algorithm}

\subsection{Cluster-guided (ensemble) trees}\label{sec:CG_methods}

The DICS algorithm (\Cref{alg:split_candidates}) precalculates candidate splits guided by clustering, and the classification tree is subsequently grown from the resulting split-candidate dictionary $\Scal$. The corresponding adaptations of the decision tree and random forest, referred to as the \emph{Cluster-Guided Classification Tree} (CGCT) and \emph{Cluster-Guided Random Forest} (CGRF), are summarized in \Cref{alg:Tree_build} (see \Cref{sec:CGCT} for details) and \Cref{alg:cgrf} (see \Cref{sec:CGRF} for details), respectively.

As for gradient boosting machine, we propose \textbf{Fast} \textbf{C}lassification \textbf{G}radient \textbf{B}oosting \textbf{M}achine (\textit{FastC-GBM}), which leverages candidate splits generated by DICS and incorporates the Gradient-based One-Side Sampling (GOSS) technique~\cite{ke2017lightgbm} and column (feature) subsampling technique~\cite{chen2016xgboost,ke2017lightgbm}. In particular, when evaluating the quality of a split $(f,\tau)$, the computational cost is further reduced by sampling a subset of samples with small gradient magnitudes and subsampling features. The full procedure of FastC-GBM is summarized in \Cref{alg:cgxgb_goss} (see \Cref{sec:FastC-GBM} for details).

\subsection{Implementation and complexity}\label{sec:complexity}

The candidate splits generated by DICS depend on the results of K-means. Since DICS only requires stable centroids, we adopt mini-batch K-means \cite{sculley2010web} together with K-means++ initialization \cite{arthur2007k}. Under this strategy, DICS introduces a one-time computational overhead of order $\mathcal{O}(TKbP)$, where $T$ is the number of iterations, $b$ is the mini-batch size, $K$ is the number of clusters, and $P$ is the feature dimension. This one-time is relative small compared to the split finding stage. 

For split finding, DICS requires a computational cost of $\mathcal{O}(mPU)$, where $m$ is the number of centroid pairs and $U$ denotes the number of operations needed to evaluate the quality of a candidate split. In contrast, an exhaustive search incurs a cost of $\mathcal{O}(NPU)$, while histogram-based search reduces this to $\mathcal{O}(HPU)$, where $H$ is the number of bins, typically set to 255 to balance efficiency and accuracy~\cite{ke2017lightgbm}. Consequently, DICS significantly reduces the per-tree training cost to a factor of $m$. When the number of classes $K$ is small (which is often the case), $m = K(K-1)/2$ remains small; for example, $m=10$ when $K=5$. When $K$ is large, our sensitivity experiments (see Appendix~\ref{sec:app_sensitivity}) show that,
in practice, a small value of $m$ is still sufficient to achieve comparable
accuracy. We further observe that accuracy remains stable across a range of tree depths ($d$) and mini-batch sizes ($b$).

\section{Split-Level Stability Analysis}
In this section, we provide a theoretical analysis of the proposed DICS algorithm. Our main result in \Cref{thm:main} shows that the gain of the optimal split selected by DICS is close to that of the optimal split selected by a standard decision tree (DT). In particular, the difference between these two gains is bounded by $\mathcal{O}(1/\sqrt{N})$, implying that as the number of samples $N$ increases, the gap vanishes asymptotically. 

Let the set of candidate splits of DT be given by $\Scal_{\mathrm{DT}}=\{(f,\bar\tau_i^{(f)}):i=1,\ldots,N-1\}$ where $\bar\tau_i^{(f)}$ are the midpoints of the sorted observation values along feature $f$. The set of candidate splits generated by DICS is denoted by $\Scal_{\mathrm{DICS}}=\{(f,\tilde\tau_i^{(f)}):i=1,\ldots,m\}$ where $\tilde\tau_i^{(f)}$ are the thresholds produced by the DICS procedure. 
For any split $(f,\tau)$, it defines a partition $\mathcal{P}_{f,\tau} = (\mathcal{I}_L, \mathcal{I}_R)$, where $\mathcal{I}_L = \{ i \in [N] : \Xbf_{f,i} < \tau \}$ and $\mathcal{I}_R = \{ i \in [N] : \Xbf_{f,i} \ge \tau \}$ are the corresponding index sets with sizes $n_L=|\Ical_L|$ and $n_R=|\Ical_R|$. 
The optimal splits for DT and DICS are defined as $(\bar f^*,\bar\tau^*)=\arg\min_{(f,\tau)\in\Scal_{\mathrm{DT}}}Q(f,\tau)$ and $(\tilde f^*,\tilde\tau^*)=\arg\min_{(f,\tau)\in\Scal_{\mathrm{DICS}}}Q(f,\tau)$, respectively. Consequently, the corresponding gains for DT and DICS are given by $G(\Ical)-Q(\bar f^*,\bar\tau^*)$ and $G(\Ical)-Q(\tilde f^*,\tilde\tau^*)$, respectively. For notational convenience, we denote $\bar Q^*=Q(\bar f^*,\bar\tau^*)$ and $\tilde Q^*=Q(\tilde f^*,\tilde\tau^*)$. 

\begin{theorem}\label{thm:main}
For every $f$, let $\bar\tau_*^{(f)}=\arg\min_{\tau\in\{\bar\tau_i^{(f)},1\le i\le N-1\}}Q(f,\tau)$ be the optimal threshold for DT and $\tilde\tau_*^{(f)}=\arg\min_{\tau\in\{\tilde\tau_i^{(f)},1\le i\le m\}}Q(f,\tau)$ be the optimal threshold for DICS, where $Q(f,\tau)$ is the Gini impurity function. Suppose that $\alpha=\min_f\{\frac{n_L^{(f)}}{N},\frac{n_R^{(f)}}{N}\}>0$. Further assume that a density regularity condition holds, i.e., for every feature $f$, the marginal density $p(x^{(f)})$ is bounded above by a constant $M$ in a neighborhood of the optimal split ${\tau}_*^{(f)}$. Then
$\delta = d_H\!\left(\mathcal{P}_{f,\bar{\tau}_*^{(f)}}, \mathcal{P}_{f,\tilde{\tau}_*^{(f)}}\right)
= \mathcal{O}_P(\sqrt{N})$.
Moreover, the following bound holds:
\begin{align*}
|\bar{Q}^* - \tilde{Q}^*|
< \left(8 + \frac{1}{\alpha(1-\alpha)} \cdot \frac{\delta}{N}\right)\frac{\delta}{N}.
\end{align*}
\end{theorem}

\begin{proof}
The proof is provided in Appendix~\ref{sec:proof}.
\end{proof}

Here, $d_H(\Pcal_1,\Pcal_2)$ in the above theorem denotes the Hamming distance between two partitions, defined as $d_H(\Pcal_1,\Pcal_2)=|\{i\in[N]:z_i\neq z'_i\}|$ where each partition $\mathcal{P} = (\mathcal{I}_L, \mathcal{I}_R)$ is associated with a label vector $z \in \{L,R\}^N$ defined by $z_i = L$ if $i \in \mathcal{I}_L$ and $z_i = R$ if $i \in \mathcal{I}_R$. 

\textbf{Remark 1.}
The assumption in \Cref{thm:main} states that the data distribution is well-behaved in a neighborhood of the decision boundary. Such a condition is standard in statistical learning theory, as it rules out degenerate cases in which the marginal density vanishes or concentrates excessively near the optimal split. 
Under this assumption, \Cref{thm:main} shows that the partitions induced by the optimal thresholds of DT and DICS differ on at most $\delta = \mathcal{O}_P(\sqrt{N})$ samples.
Although $\delta$ grows with $N$, the difference between the corresponding optimal gains satisfies $\mathcal{O}_P(1/\sqrt{N})$, and therefore converges to zero as $N$ increases. We provide complementary empirical evidence for this split-level behavior in Appendix~\ref{sec:app_split_quality}. This bound is established at the root, where the DICS dictionary is estimated from the same distribution being split; at deeper nodes, the global dictionary is not guaranteed to satisfy the same alignment, and the depth-wise diagnostic in Appendix~\ref{sec:app_split_quality} (Table~\ref{tab:cifar10_depth_gain}) shows the cumulative retained gain ratio $R_d$ declining from $100\%$ at the root to $97.18\%$ by depth~8.

\textbf{Remark 2.} \Cref{thm:main} only establishes that the gain achieved by DICS is asymptotically close to that of DT, which does not imply that one method is necessarily better than the other, as the relationship between gain and classification accuracy is nontrivial. This theoretical result is also supported by our numerical experiments.

\section{Numerical Experiments}\label{sec:numerical}

We evaluate the proposed split-generation method, DICS, across several decision tree–based models, including \textit{(1) standard decision trees (DT), (2) random forests (RF), and (3) gradient boosting systems}. The comparisons are conducted on both synthetic and real-world datasets. All results are obtained on a machine equipped with an Apple M4 Max chip and 64GB of unified memory.

\textbf{Baseline Methods.} As relatively few studies integrate clustering into tree construction, we restrict our comparison in the DT setting to BDTKS \cite{wang2020linear}. We exclude Clus-DTI \cite{barros2011clustering}, as it alters the underlying tree objective. Moreover, since BDTKS applies the K-means algorithm recursively at each node, it becomes computationally expensive for large datasets and is not readily compatible with ensemble methods. Therefore, we do not include it in the RF and boosting settings.

\textbf{Parameters Setup.} We use the Gini index as the splitting criterion to compute the quality $Q(f,\tau)$. The maximum tree depth is set to $D = 8$ for DT and RF, and $D = 3$ for boosting-based methods. The number of sampled candidate splits is fixed at $k = 100$. For $K \le 5$, we set the number of centroid pairs to $m = K(K-1)/2$; otherwise, we use $m = 25$. The batch size for mini-batch K-means is set to $512$.

\begin{table*}[htbp]
\centering
\footnotesize
\setlength{\tabcolsep}{1.7pt}
\begin{tabular}{cc cccc cccc}
\toprule
\multirow{2}{*}{$(N,P)$} & \multirow{2}{*}{Class}
& \multicolumn{3}{c}{Time (sec)$\downarrow$}
& \multicolumn{1}{c}{Speedup$\uparrow$}
& \multicolumn{3}{c}{Test Acc$\uparrow$} \\
\cmidrule(lr){3-5} \cmidrule(lr){6-6} \cmidrule(lr){7-9}
& & DT & BDTKS & CGCT
& $t_{\mathrm{DT}}/t_{\mathrm{CGCT}}$
& DT & BDTKS & CGCT \\
\midrule

\multirow{3}{*}{10000,1000}
& 3  & 2.47 (0.03)   & 30.16      & \textbf{0.19 (0.02)}
     & 13.00
     & \textbf{0.62 (0.00)} & 0.54 & 0.60 (0.00) \\
& 5  & 2.53 (0.01)   & 34.36    & \textbf{0.25 (0.02)}
     & 10.12
     & \textbf{0.54 (0.00)} & 0.39  & 0.52 (0.00) \\
& 10 & 2.52 (0.00)   & 35.68   & \textbf{0.29 (0.02)}
     & 8.69
     & \textbf{0.43 (0.00)} & 0.32 & 0.41 (0.00) \\
\cmidrule(lr){1-9}

\multirow{3}{*}{5000,7000}
& 3  & 8.36 (0.10)   & 29.85     & \textbf{0.58 (0.05)}
     & 14.41
     & \textbf{0.61 (0.01)} & 0.57 & 0.59 (0.01) \\
& 5  & 8.14 (0.03) & 29.44   & \textbf{0.91 (0.04)}
     & 8.94
     & \textbf{0.47 (0.01)} & 0.38  & 0.45 (0.00) \\
& 10 & 8.27 (0.02)  & 30.93  & \textbf{1.04 (0.08)}
     & 7.95
     & \textbf{0.31 (0.00)} & 0.27 & \textbf{0.31 (0.01)} \\
\cmidrule(lr){1-9}

\multirow{3}{*}{20000,5000}
& 3  & 26.89 (0.26) & 107.82 & \textbf{1.20 (0.07)}
     & 22.40
     & \textbf{0.53 (0.00)} & 0.47 & 0.52 (0.00) \\
& 5  & 27.17 (0.04)  & 110.24   & \textbf{1.77 (0.07)}
     & 15.35
     & \textbf{0.47 (0.00)} & 0.38  & 0.46 (0.00) \\
& 10 & 27.98 (0.05)  & 114.27 & \textbf{1.99 (0.09)}
     & 14.06
     & 0.48 (0.00) & 0.39 & \textbf{0.49 (0.00)} \\
\bottomrule
\end{tabular}
\caption{Accuracy and training time (in seconds) for DT, BDTKS, and CGCT in synthetic data. The ratio $t_{\mathrm{DT}}/t_{\mathrm{CGCT}}$ quantifies the training time improvement of CGCT relative to DT.}
\label{tab:synthetic_data_CGCT_rho05}
\end{table*}

\begin{table*}[htbp]
\centering
\footnotesize
\begin{tabular}{cc cc c cc}
\toprule
\multirow{2}{*}{$(N,P)$} & \multirow{2}{*}{Class}
& \multicolumn{2}{c}{Time (sec)$\downarrow$}
& \multicolumn{1}{c}{Speedup$\uparrow$}
& \multicolumn{2}{c}{Test Acc$\uparrow$} \\
\cmidrule(lr){3-4} \cmidrule(lr){5-5} \cmidrule(lr){6-7}
& & RF & CGRF & $t_{\mathrm{RF}}/t_{\mathrm{CGRF}}$ & RF & CGRF \\
\midrule

\multirow{3}{*}{10000,1000}
& 3  & 4.05 (0.13)  & \best{0.36 (0.04)}  & 11.25  & \best{0.64 (0.00)}  & \best{0.64 (0.00)} \\
& 5  & 4.32 (0.02)   & \best{0.41 (0.03)}  & 10.54  & \best{0.56 (0.01)}  & \best{0.56 (0.00)} \\
& 10 & 4.25 (0.04)   & \best{0.53 (0.02)}  & 8.02  & \best{0.45 (0.01)} & \best{0.45 (0.01)} \\
\cmidrule(lr){1-7}

\multirow{3}{*}{5000,7000}
& 3  & 13.27 (0.45)   & \best{0.93 (0.10)}   & 14.27 & \best{0.64 (0.00)} & \best{0.64 (0.01)} \\
& 5  & 13.18 (0.05)  & \best{1.14 (0.05)}  & 11.56 & \best{0.53 (0.00)} & 0.52 (0.00) \\
& 10 & 13.17 (0.04)     & \best{1.66 (0.08)}   & 7.93 & \best{0.36 (0.00)} & \best{0.36 (0.00)} \\
\cmidrule(lr){1-7}

\multirow{3}{*}{20000,5000}
& 3  & 44.47 (1.45)   & \best{1.49 (0.04)}  & 29.84 & \best{0.56 (0.00)} & 0.54 (0.00) \\
& 5  & 48.26 (1.07)  & \best{1.67 (0.10)}  & 28.90 & \best{0.49 (0.00)} & \best{0.49 (0.00)} \\
& 10 & 53.25 (0.60)  & \best{2.09 (0.06)}  & 25.48 & \best{0.52 (0.00)} & \best{0.52 (0.00)} \\
\bottomrule
\end{tabular}
\caption{Accuracy and training time (in seconds) for RF and CGRF in synthetic data. The ratio $t_{\mathrm{RF}}/t_{\mathrm{CGRF}}$ quantifies the training time improvement of CGRF relative to RF. }
\label{tab:synthetic_data_CGRF_rho05}
\end{table*}

\begin{table*}[htbp]
\centering
\setlength{\tabcolsep}{4.0pt}
\resizebox{\linewidth}{!}{
\begin{tabular}{cc ccc cc ccc}
\toprule
\multirow{2}{*}{$(N,P)$} & \multirow{2}{*}{Class}
& \multicolumn{3}{c}{Time (sec)$\downarrow$}
& \multicolumn{2}{c}{Speedup$\uparrow$}
& \multicolumn{3}{c}{Test Acc$\uparrow$} \\
\cmidrule(lr){3-5} \cmidrule(lr){6-7} \cmidrule(lr){8-10}
& & XGBoost & LightGBM & FastC-GBM
& $\rho_1$
& $\rho_2$
& XGBoost & LightGBM & FastC-GBM \\
\midrule

\multirow{3}{*}{10000,1000}
& 3  & 1.62 (0.03)  & 1.16 (0.06) & \best{0.32 (0.00)} & 5.06 & 3.62 & \best{0.68 (0.00)} & \best{0.68 (0.0)} & 0.65 (0.01) \\
& 5  & 2.66 (0.07)  & 1.81 (0.07) & \best{0.35 (0.00)} & 7.60 & 5.17 & 0.59 (0.01) & \best{0.60 (0.01)} & 0.59 (0.01) \\
& 10 & 5.03 (0.02) & 3.31 (0.0) & \best{0.51 (0.01)} & 9.86 & 6.49 & \best{0.47 (0.01)} & 0.46 (0.01) & 0.45 (0.01) \\
\cmidrule(lr){1-10}

\multirow{3}{*}{5000,7000}
& 3  & 6.62 (0.06) & 3.09 (0.03) & \best{0.24 (0.00)} & 27.58 & 12.87 & \best{0.63 (0.00)} & \best{0.63 (0.02)} & 0.61 (0.02) \\
& 5  & 11.25 (0.09) & 4.96 (0.07) & \best{0.28 (0.00)} & 40.18 & 17.71 & \best{0.51 (0.02)} & 0.50 (0.02) & 0.50 (0.02) \\
& 10 & 18.06 (0.45) & 9.80 (0.33) & \best{0.36 (0.00)} & 50.17 & 27.22 & \best{0.35 (0.01)} & \best{0.35 (0.01)} & 0.34 (0.01) \\
\cmidrule(lr){1-10}

\multirow{3}{*}{20000,5000}
& 3  & 8.52 (0.30)  & 5.28 (0.18) & \best{0.66 (0.02)} & 12.90 & 8.00 & \best{0.57 (0.01)} & \best{0.57 (0.01)} & 0.56 (0.01) \\
& 5  & 15.57 (0.61)  & 8.41 (0.24) & \best{0.48 (0.00)} & 32.44 & 17.52 & \best{0.51 (0.01)} & \best{0.51 (0.01)} & 0.49 (0.01) \\
& 10 & 31.41 (0.27) & 14.72 (0.12) & \best{1.01 (0.02)} & 31.10 & 14.57 & \best{0.53 (0.01)} & \best{0.53 (0.01)} & \best{0.53 (0.01)} \\
\bottomrule
\end{tabular}
}
\caption{Accuracy and training time (in seconds) for XGBoost, LightGBM and FastC-GBM in synthetic data. The ratios $\rho_1=t_{\mathrm{XGBoost}}/t_{\mathrm{FastC-GBM}}$ and $\rho_2=t_{\mathrm{LightGBM}}/t_{\mathrm{FastC-GBM}}$ quantify the training time improvement of FastC-GBM relative to XGBoost and LightGBM, respectively. }
\label{tab:xgb_lgb_cgxgb_rho05}
\end{table*}

\subsection{Synthetic data}

\textbf{Formulation.} To evaluate the performance of the tree-based methods, we generate a multiclass synthetic dataset in which the class label depends on a mixture of linear and nonlinear feature effects.
For each simulation, we sample $N$ observations $\{(\xbm_i,y_i)\}_{i=1}^N$ as follows. The $P$-dimensional feature vector $X \in \mathbb{R}^P$ is drawn from a multivariate normal distribution $X \sim \mathcal{N}(0,\Sigma)$, where $\Sigma_{ii} = 1$ for all $i \in [P]$ and $\Sigma_{ij}=\rho$ for all $i \neq j$.
Given $X$, the label $Y$ follows a categorical distribution, $Y\mid X\sim\mathrm{Categorical}(\theta_1,\dots,\theta_K)$, where  
\[\theta_c=P(Y=c|X):=\frac{\exp(s_c)}{\sum_{c'=1}^K\exp(s_{c'})},\quad \forall c=1,\ldots,K\]
with
\begin{align*}
    s_c=\sum_{j \in \mathcal{S}_1} a_j X_j+\sum_{j \in \mathcal{S}_2} b_j \sin\bigl(\pi X_j\bigr)
+
\sum_{j \in \mathcal{S}_3} c_j X_j^2
+
\sum_{\substack{k,l \in \mathcal{S}_4 \\ k < l}} d_{kl} X_{k} X_{l}+\eta.
\end{align*}
Here, the variable sets $\Scal_1, \Scal_2, \Scal_3,$ and $\Scal_4$ correspond to linear, sinusoidal, quadratic, and pairwise interaction effects, respectively. The collection ${\Scal_1,\Scal_2,\Scal_3,\Scal_4}$ forms a partition of $[P]$, with cardinalities given by $|\Scal_1|=\lfloor P/2 \rfloor$, $|\Scal_2|=\lfloor P/4 \rfloor$, $|\Scal_3|=\lfloor P/8 \rfloor$, and $|\Scal_4| = P - \sum_{k=1}^3 |\Scal_k|$. Moreover, the weights $a_j,b_j,c_j,d_{kl},\eta\sim\Ncal(0,1)$. 

\textbf{Results.} For the synthetic datasets, we repeat each simulation 10 times, except for BDTKS, which is run only once due to its high training cost. The mean and standard deviation (reported in parentheses) of the training time and accuracy for all methods across the three settings with $\rho=0.5$ are presented in \Cref{tab:synthetic_data_CGCT_rho05}–\ref{tab:xgb_lgb_cgxgb_rho05}. Additional results for the case $\rho = 0$ are provided in Appendix \ref{sec:app_sim}.

From \Cref{tab:synthetic_data_CGCT_rho05}, CGCT delivers an impressive 8–22x speedup over the classical DT on synthetic data, with only a marginal trade-off in accuracy. In contrast, BDTKS does not demonstrate advantages in either training speed or accuracy. In the random forest and boosting settings, \Cref{tab:synthetic_data_CGRF_rho05} and \Cref{tab:xgb_lgb_cgxgb_rho05} demonstrate that DICS substantially accelerates training with negligible impact on accuracy (within 0.02). Specifically, CGRF is 8–30x faster than RF, while FastC-GBM is 3.6–27x faster than LightGBM and 5–50x faster than XGBoost.

\begin{table*}[htbp]
\centering
\footnotesize
\setlength{\tabcolsep}{2.0pt}
\begin{tabular}{lccccccc}
\toprule
& \multicolumn{3}{c}{Train time (s)$\downarrow$}
& Speedup$\uparrow$
& \multicolumn{3}{c}{Test Acc$\uparrow$}
\\
\cmidrule(lr){2-4} \cmidrule(lr){5-5} \cmidrule(lr){6-8}
Dataset & DT & CGCT & BDTKS  & $t_{DT}/t_{CGCT}$ & DT & CGCT & BDTKS \\
\midrule
Helena         & 1.30  & \best{0.12} & 470.25     & 10.83  & \best{0.28} & 0.26 & 0.16  \\
Spambase       & 0.02  & \best{0.01}  & 5.42       & 2.00   & \best{0.95} & 0.93 & 0.79  \\
Santander      & 17.99 & \best{0.83}  & 714.26    & 21.67  & \best{0.91} & 0.90 & 0.85 \\
CIFAR-10       & 42.63 & \best{3.33}  & 533.91    & 12.80  & \best{0.34} & 0.31 & 0.29  \\
MNIST          & 3.27  & \best{0.99}  & 124.42     & 3.30   & 0.83 & 0.79 & \textbf{0.91} \\
Fashion-MNIST  & 7.27  & \best{1.07}  & 176.54     & 6.80   & \best{0.80} & 0.76 & \best{0.80}  \\
\bottomrule
\end{tabular}
\caption{Accuracy and training time (in seconds) for DT, BDTKS, and CGCT in real datasets. The ratio $t_{\mathrm{DT}}/t_{\mathrm{CGCT}}$ quantifies the training time improvement of CGCT relative to DT.}
\label{tab:runtime_accuracy_dt_bdtks_cgdt}
\end{table*}

\subsection{Real Datasets}

\textbf{Real Datasets.} We use six real-world datasets spanning a variety of domains and scales, including Helena~\cite{openml2018helena}, Spambase~\cite{opml-spambase}, Santander~\cite{santander_customer_transaction_2019}, CIFAR-10~\cite{krizhevsky2009learning}, MNIST~\cite{lecun2010mnist}, and Fashion-MNIST~\cite{xiao2017fashion}. These datasets cover both tabular and image classification tasks, enabling a comprehensive evaluation of different model behaviors across heterogeneous data types. A summary of these datasets is given in \Cref{tab:summary_real} (Appendix).

Helena is a high-dimensional biological dataset from OpenML, commonly used for benchmarking feature selection and classification methods. Spambase is a classic UCI dataset for email spam detection based on word and character frequencies. The Santander dataset consists of anonymized customer transaction features for binary classification in a financial setting.
CIFAR-10 is a standard image recognition benchmark containing 10 object categories with natural images, while MNIST is a widely used handwritten digit dataset for digit classification. Fashion-MNIST is a more challenging variant of MNIST that contains grayscale images of clothing items, providing a modern replacement for digit recognition benchmarks.

\textbf{Results.}
The training time and accuracy of all methods across the three settings on real-world datasets are reported in \Cref{tab:runtime_accuracy_dt_bdtks_cgdt}--\Cref{tab:xgb_lgb_cgxgb_real}. Overall, tree-based algorithms enhanced with DICS achieve substantially improved training efficiency in these datasets. 
For the decision tree comparison in \Cref{tab:runtime_accuracy_dt_bdtks_cgdt}, CGCT is approximately 2-21x faster than DT, while its accuracy decreases slightly more than in the synthetic setting. In contrast, BDTKS is significantly slower and exhibits unstable performance across datasets.
From \Cref{tab:rf_vs_cgrf}, DICS accelerates the RF training by approximately 2.3-12.8x, while the accuracy drop remains small (up to $0.02$). For boosting methods comparison in \Cref{tab:xgb_lgb_cgxgb_real}, the proposed FastC-GBM is consistently faster than LightGBM (up to 9.5x) and XGBoost (up to 18.75x).
The only exception is the Santander dataset, where the performance degradation is likely due to implementation-level memory management issues on a large-scale dataset ($N = 200$k), rather than the algorithm itself. Finally, we observe that LightGBM performs unusually poorly on the Helena dataset, whereas FastC-GBM maintains stable performance, suggesting improved robustness compared to LightGBM, which is known to rely on a more aggressive splitting strategy that may affect stability in certain regimes. 

\begin{table*}[htbp]
\centering
\small
\setlength{\tabcolsep}{3.0pt}
\begin{tabular}{lccccc}
\toprule
& \multicolumn{2}{c}{Train time (s)$\downarrow$}
& Speedup$\uparrow$
& \multicolumn{2}{c}{Test Acc$\uparrow$} \\
\cmidrule(lr){2-3} \cmidrule(lr){4-4} \cmidrule(lr){5-6}
Dataset & RF & CGRF & $t_{\mathrm{RF}}/t_{\mathrm{CGRF}}$ & RF & CGRF \\
\midrule
Helena        & 1.43 (0.12) & \textbf{0.22 (0.02)} & 6.50 & \textbf{0.29} (0.00) & 0.28 (0.00) \\
Spambase      & 0.09 (0.00) & \best{0.03 (0.01)}& 3.00 & \textbf{0.95 (0.00)} &0.94 (0.00) \\
Santander     & 22.04 (0.46) & \textbf{3.12 (0.04)} & 7.06 & \textbf{0.90 (0.00)} & \textbf{0.90 (0.00)} \\
CIFAR-10      & 55.96 (2.47) & \textbf{4.37 (0.12)} & 12.80 & \textbf{0.40 (0.01)} & 0.38 (0.01) \\
MNIST & 5.24 (0.25) & \textbf{2.28 (0.02)} & 2.30 & \textbf{0.93 (0.00)} & 0.92 (0.00) \\
Fashion-MNIST & 10.67 (0.05) & \textbf{2.40 (0.03)} & 4.45 & \textbf{0.83 (0.00)} & 0.82 (0.00) \\
\bottomrule
\end{tabular}
\caption{Accuracy and training time (in seconds) for RF and CGRF in real datasets. The ratio $t_{\mathrm{RF}}/t_{\mathrm{CGRF}}$ quantifies the training time improvement of CGRF relative to RF. }
\label{tab:rf_vs_cgrf}
\end{table*}

\begin{table*}[htbp]
\centering
\setlength{\tabcolsep}{3.5pt}
\resizebox{\linewidth}{!}{
\begin{tabular}{lcccccccc}
\toprule
& \multicolumn{3}{c}{Train time (s)$\downarrow$}
& \multicolumn{2}{c}{Speedup$\uparrow$}
& \multicolumn{3}{c}{Test Acc$\uparrow$} \\
\cmidrule(lr){2-4} \cmidrule(lr){5-6} \cmidrule(lr){7-9}
Dataset & XGBoost & LightGBM & FastC-GBM & $\rho_1$ & $\rho_2$ & XGBoost & LightGBM & FastC-GBM \\
\midrule
Helena        & 13.01 (0.14) & 13.38 (0.15) & \best{5.48 (0.26)} & 2.37 & 2.44 & \best{0.35 (0.00)} & 0.23 (0.00) & 0.32 (0.00) \\
Spambase      & 0.20 (0.00)  & 0.16 (0.01)           & \best{0.06 (0.00)} & 3.34 & 2.67 & 0.95 (0.00) & \best{0.95 (0.00)} & 0.94 (0.00) \\
Santander     & \best{0.91 (0.03)} & 1.09 (0.02) & 1.12 (0.01) & 0.81 & 0.97 & \best{0.90 (0.00)} & \best{0.90 (0.00)} & \best{0.90 (0.00)} \\
CIFAR-10      & 52.50 (0.86) & 26.57 (0.29) & \best{2.80 (0.04)} & 18.75 & 9.49 & 0.45 (0.00) & \best{0.48 (0.00)} & 0.46 (0.00) \\
MNIST         & 29.42 (0.07) & 4.33 (0.01) & \best{2.55 (0.04)} & 11.54 & 1.70 & 0.94 (0.00) & \best{0.95 (0.00)} & \best{0.95 (0.00)} \\
Fashion-MNIST & 22.24 (0.11) & 10.93 (0.10) & \best{2.85 (0.04)} & 7.80 & 3.83 & 0.86 (0.00) & \best{0.87 (0.00)} & 0.86 (0.00) \\
\bottomrule
\end{tabular}
}
\caption{Accuracy and training time (in seconds) for XGBoost, LightGBM and FastC-GBM in real datasets. The ratios $\rho_1=t_{\mathrm{XGBoost}}/t_{\mathrm{FastC-GBM}}$ and $\rho_2=t_{\mathrm{LightGBM}}/t_{\mathrm{FastC-GBM}}$ quantify the training time improvement of FastC-GBM relative to XGBoost and LightGBM, respectively. }
\label{tab:xgb_lgb_cgxgb_real}
\end{table*}

\section{Conclusion and Future Directions}\label{sec:conclusion}

We propose DICS, a novel approach for generating data-informed candidate splits in tree-based methods, resulting in a substantially reduced search space. Our theoretical analysis shows that, asymptotically, DICS yields candidate splits of comparable quality to those of classical decision trees, and therefore achieves similar predictive performance. Empirical results further demonstrate that DICS maintains comparable accuracy while significantly improving computational efficiency.
Currently, the proposed approach focuses on accelerating classification tasks, which serves as the main \textit{limitation}. An important direction for future work is to extend this framework by incorporating data-informed priors to similarly improve efficiency in regression settings.

\medskip

{
\small
\bibliographystyle{unsrt}
\bibliography{references}

@ARTICLE{Entropy,
  author={Wang, Qing Ren and Suen, Ching Y.},
  journal={IEEE Transactions on Pattern Analysis and Machine Intelligence}, 
  title={Analysis and Design of a Decision Tree Based on Entropy Reduction and Its Application to Large Character Set Recognition}, 
  year={1984},
  volume={PAMI-6},
  number={4},
  pages={406-417},
  doi={10.1109/TPAMI.1984.4767546}}

@inproceedings{mehta1996sliq,
  title={{SLIQ}: A fast scalable classifier for data mining},
  author={Mehta, Manish and Agrawal, Rakesh and Rissanen, Jorma},
  booktitle={International conference on extending database technology},
  pages={18--32},
  year={1996},
  organization={Springer}
}

@article{loh1997split,
  title={Split selection methods for classification trees},
  author={Loh, Wei-Yin and Shih, Yu-Shan},
  journal={Statistica sinica},
  pages={815--840},
  year={1997},
  publisher={JSTOR}
}

@inproceedings{gehrke1998rainforest,
  title={Rainforest-a framework for fast decision tree construction of large datasets},
  author={Gehrke, Johannes and Ramakrishnan, Raghu and Ganti, Venkatesh},
  booktitle={VLDB},
  volume={98},
  pages={416--427},
  year={1998}
}

@inproceedings{chickering2001efficient,
  title={Efficient determination of dynamic split points in a decision tree},
  author={Chickering, David Maxwell and Meek, Christopher and Rounthwaite, Robert},
  booktitle={Proceedings 2001 IEEE international conference on data mining},
  pages={91--98},
  year={2001},
  organization={IEEE}
}

@article{breiman2001random,
  title={Random forests},
  author={Breiman, Leo},
  journal={Machine learning},
  volume={45},
  number={1},
  pages={5--32},
  year={2001},
  publisher={Springer}
}

@article{geurts2006extratrees,
  author  = {Geurts, Pierre and Ernst, Damien and Wehenkel, Louis},
  title   = {Extremely Randomized Trees},
  journal = {Machine Learning},
  volume  = {63},
  number  = {1},
  pages   = {3--42},
  year    = {2006},
  doi     = {10.1007/s10994-006-6226-1}
}

@article{ke2017lightgbm,
title={{LightGBM}: A highly efficient gradient boosting decision tree},
  author={Ke, Guolin and Meng, Qi and Finley, Thomas and Wang, Taifeng and Chen, Wei and Ma, Weidong and Ye, Qiwei and Liu, Tie-Yan},
  journal={Advances in neural information processing systems},
  volume={30},
  year={2017}
}

@inproceedings{chen2016xgboost,
  title={{XGBoost}: {A} scalable tree boosting system},
  author={Chen, Tianqi and Guestrin, Carlos},
  booktitle={Proceedings of the 22nd acm sigkdd international conference on knowledge discovery and data mining},
  pages={785--794},
  year={2016}
}

@inproceedings{macqueen1967multivariate,
  title={Some methods of classification and analysis of multivariate observations},
  author={McQueen, James B.},
  booktitle={Proc. of 5th Berkeley Symposium on Math. Stat. and Prob.},
  pages={281--297},
  year={1967}
}

@inproceedings{barros2011clustering,
  title={A clustering-based decision tree induction algorithm},
  author={Barros, Rodrigo C and de Carvalho, Andr{\'e} CPL R and Basgalupp, Marcio R and Quiles, Marcos G},
  booktitle={2011 11th International Conference on Intelligent Systems Design and Applications},
  pages={543--550},
  year={2011},
  organization={IEEE}
}

@article{wang2020linear,
  title={A linear multivariate binary decision tree classifier based on {K-means} splitting},
  author={Wang, Fei and Wang, Quan and Nie, Feiping and Li, Zhongheng and Yu, Weizhong and Ren, Fuji},
  journal={Pattern Recognition},
  volume={107},
  pages={107521},
  year={2020},
  publisher={Elsevier}
}

@inproceedings{lin2020generalized,
  title={Generalized and scalable optimal sparse decision trees},
  author={Lin, Jimmy and Zhong, Chudi and Hu, Diane and Rudin, Cynthia and Seltzer, Margo},
  booktitle={Proceedings of the 37th International Conference on Machine Learning (ICML'20)},
  volume={119},
  pages={6150--6160},
  year={2020},
  organization={PMLR}
}

@inproceedings{mctavish2022fast,
  title={Fast sparse decision tree optimization via reference ensembles},
  author={McTavish, Hayden and Zhong, Chudi and Achermann, Reto and Karimalis, Ilias and Chen, Jacques and Rudin, Cynthia and Seltzer, Margo},
  booktitle={Proceedings of the AAAI Conference on Artificial Intelligence},
  volume={36},
  number={9},
  pages={9604--9613},
  year={2022}
}

@inproceedings{hu2019osdt,
  author    = {Hu, Xiyang and Rudin, Cynthia and Seltzer, Margo},
  title     = {Optimal Sparse Decision Trees},
  booktitle = {Advances in Neural Information Processing Systems},
  volume    = {32},
  pages     = {7265--7273},
  year      = {2019},
  publisher = {Curran Associates, Inc.},
  url       = {https://proceedings.neurips.cc/paper/2019/hash/ac52c626afc10d4075708ac4c778ddfc-Abstract.html}
}

@inproceedings{zhang2023optimal,
  title={Optimal Sparse Regression Trees},
  author={Zhang, Rui and Xin, Rui and Seltzer, Margo and Rudin, Cynthia},
  booktitle={Proceedings of the AAAI Conference on Artificial Intelligence},
  volume={37},
  number={9},
  pages={11172--11180},
  year={2023}
}

@inproceedings{aglin2020learning,
  author    = {Ga{\"{e}}l Aglin and Siegfried Nijssen and Pierre Schaus},
  title     = {Learning Optimal Decision Trees Using Caching Branch-and-Bound Search},
  booktitle = {Proceedings of the AAAI Conference on Artificial Intelligence},
  volume    = {34},
  number    = {04},
  pages     = {3146--3153},
  year      = {2020},
  doi       = {10.1609/aaai.v34i04.5712}
}

@misc{openml2018helena,
  author = {OpenML},
  title = {{HELENA} Dataset},
  year = {2018},
  howpublished = {\url{https://www.openml.org/d/41169}}
}

@article{krizhevsky2009learning,
  title={Learning Multiple Layers of Features from Tiny Images},
  author={Krizhevsky, Alex},
  year={2009},
  journal={University of Toronto}
}

@article{lecun2010mnist,
  title   = {Gradient-based learning applied to document recognition},
  author  = {LeCun, Yann and Bottou, L{\'e}on and Bengio, Yoshua and Haffner, Patrick},
  journal = {Proceedings of the IEEE},
  volume  = {86},
  number  = {11},
  pages   = {2278--2324},
  year    = {1998},
  doi     = {10.1109/5.726791}
}

@article{xiao2017fashion,
  title={{Fashion-MNIST}: a novel image dataset for benchmarking machine learning algorithms},
  author={Xiao, Han and Rasul, Kashif and Vollgraf, Roland},
  journal={arXiv preprint arXiv:1708.07747},
  year={2017}
}

@misc{santander_customer_transaction_2019,
  author = {Santander},
  title = {Santander Customer Transaction Prediction},
  year = {2019},
  publisher = {Kaggle},
  journal = {Kaggle},
  howpublished = {\url{https://www.kaggle.com}}
}

@misc{opml-spambase,
  author = {Mark Hopkins and Erik Reeber and George Forman and Jaap Suermondt},
  title = {Spambase},
  year = {1999},
  howpublished = {UCI Machine Learning Repository},
  note = {OpenML dataset ID 44},
  url = {https://www.openml.org/d/44}
}

@book{duda2006pattern,
  title     = {Pattern Classification},
  author    = {Duda, Richard O. and Hart, Peter E. and Stork, David G.},
  edition   = {2},
  year      = {2001},
  publisher = {Wiley-Interscience},
  address   = {New York},
  isbn      = {9780471056690}
}

@book{bishop2006pattern,
  title={Pattern recognition and machine learning},
  author={Bishop, Christopher M. and Nasrabadi, Nasser M.},
  volume={4},
  year={2006},
  publisher={Springer}
}

@article{cover1967nearest,
  title={Nearest neighbor pattern classification},
  author={Cover, Thomas and Hart, Peter},
  journal={IEEE transactions on information theory},
  volume={13},
  number={1},
  pages={21--27},
  year={1967},
  publisher={IEEE}
}

@article{laurent1976constructing,
  title={Constructing optimal binary decision trees is {NP}-complete},
  author={Laurent, Hyafil and Rivest, Ronald L.},
  journal={Information processing letters},
  volume={5},
  number={1},
  pages={15--17},
  year={1976}
}

@article{voronoi1908nouvelles,
  title={Nouvelles applications des param{\`e}tres continus {\`a} la th{\'e}orie des formes quadratiques. Deuxi{\`e}me m{\'e}moire. Recherches sur les parall{\'e}llo{\`e}dres primitifs.},
  author={Voronoi, Georges},
  journal={Journal f{\"u}r die reine und angewandte Mathematik (Crelles Journal)},
  volume={1908},
  number={134},
  pages={198--287},
  year={1908},
  publisher={De Gruyter Berlin, New York}
}

@article{li2007mcrank,
  title={{McRank}: Learning to rank using multiple classification and gradient boosting},
  author={Li, Ping and Wu, Qiang and Burges, Christopher},
  journal={Advances in neural information processing systems},
  volume={20},
  year={2007}
}

@book{breiman2017classification,
  title={Classification and regression trees},
  author={Breiman, Leo and Friedman, Jerome and Olshen, Richard A and Stone, Charles J},
  year={2017},
  publisher={Chapman and Hall/CRC},
  journal = {Cytometry},
  volume = {8},
  pages = {534-535},
  doi = {https://doi.org/10.1002/cyto.990080516},
  url = {https://onlinelibrary.wiley.com/doi/abs/10.1002/cyto.990080516}
}

@article{friedman2001greedy,
  title={Greedy function approximation: a gradient boosting machine},
  author={Friedman, Jerome H.},
  journal={Annals of statistics},
  pages={1189--1232},
  year={2001},
  publisher={JSTOR}
}

@book{gini1912variabilita,
  title={Variabilit{\`a} e mutabilit{\`a}: contributo allo studio delle distribuzioni e delle relazioni statistiche},
  author={Gini, Corrado},
  year={1912},
  publisher={Tipogr. di P. Cuppini}
}

@inproceedings{arthur2007k,
author = {Arthur, David and Vassilvitskii, Sergei},
title = {k-means++: the advantages of careful seeding},
year = {2007},
isbn = {9780898716245},
publisher = {Society for Industrial and Applied Mathematics},
address = {USA},
booktitle = {Proceedings of the Eighteenth Annual ACM-SIAM Symposium on Discrete Algorithms},
pages = {1027–1035},
numpages = {9},
location = {New Orleans, Louisiana},
series = {SODA '07}
}

@inproceedings{sculley2010web,
  title={Web-scale k-means clustering},
  author={Sculley, David},
  booktitle={Proceedings of the 19th international conference on World wide web},
  pages={1177--1178},
  year={2010}
}
}

\appendix

\section{Technical appendices and supplementary material}
In this Appendix, we provide the full details of the tree-based algorithms incorporating the DICS algorithm in \Cref{sec:alg_details}. The proof of the main theorem is presented in \Cref{sec:proof}. We further include a parameter sensitivity analysis in \Cref{sec:app_sensitivity}, empirical validation of DICS Split Quality in \Cref{sec:app_split_quality} as well as additional experimental results on synthetic data in \Cref{sec:app_sim}.

\subsection{Algorithmic Details}\label{sec:alg_details}

\begin{algorithm}[h]
\caption{Cluster-Guided Classification Tree (CGCT)}
\label{alg:Tree_build}
\begin{algorithmic}[1]

\REQUIRE Feature matrix $\Xbf \in \mathbb{R}^{P \times N}$, labels $\ybm \in [K]^N$, split-candidate dictionary $\mathcal{S}$ (Alg.~\ref{alg:split_candidates}), maximum depth $D$, number of sampled candidates $k$.
\ENSURE Decision tree $T$
\vspace{0.5em}
\STATE \textbf{Function} \textsc{BuildTree}$(\Xbf, \ybm, d)$
\IF{$d \ge D$ \textbf{or} $|\ybm| \le 2$ \textbf{or} $\ybm$ is pure}
    \STATE \textbf{return} leaf node with majority class in $\ybm$
\ENDIF

\STATE \textbf{(Split selection)}
\STATE Sample a subset $\mathcal{C} \subset \mathcal{S}$ with $|\mathcal{C}| = k$ (without replacement)
\STATE $(f^*, \tau^*) \gets \arg\min_{(f,\tau)\in\mathcal{C}} Q(f,\tau)$
\STATE $\Scal\gets\Scal/\{(f^*,\tau^*)\}$

\STATE \textbf{(Partition)}
\STATE $\Xbf_L, \ybm_L \gets \{(\xbm_i,y_i) : \Xbf_{f^*,i} < \tau^*\}$
\STATE $\Xbf_R, \ybm_R \gets \{(\xbm_i,y_i) : \Xbf_{f^*,i} \geq \tau^*\}$

\STATE \textbf{(Recursion)}
\STATE $T_L \gets \textsc{BuildTree}(\Xbf_L, \ybm_L, d+1)$
\STATE $T_R \gets \textsc{BuildTree}(\Xbf_R, \ybm_R, d+1)$

\STATE \textbf{return} node $(f^*, \tau^*, T_L, T_R)$

\vspace{0.5em}
\STATE \textbf{return} $T \gets \textsc{BuildTree}(\Xbf, \ybm, 0)$

\end{algorithmic}
\end{algorithm}

In this section, we present the detailed algorithms for tree-based methods incorporating DICS. In particular, we introduce the \emph{Cluster-Guided Classification Tree} (CGCT), the \emph{Cluster-Guided Random Forest} (CGRF), and the cluster-guided boosting method, referred to as the \textit{Fast Classification Gradient Boosting Machine} (FastC-GBM).

\subsubsection{Cluster-Guided Classification Tree (CGCT)}\label{sec:CGCT}

DICS precalculates candidate splits $\Scal$, which yields two main advantages: (a) the number of candidate splits is reduced to $mP$, significantly lowering the computational cost; (b) $\Scal$ focuses on representative thresholds that capture the primary variations in the data. As a result, it is sufficient to explore only a subset of $\Scal$ to obtain high-quality splits, avoiding exhaustive evaluation. We refer to the resulting method as the Cluster-Guided Classification Tree (CGCT), summarized in \Cref{alg:Tree_build}.

\subsubsection{Cluster-guided random forest (CGRF)}\label{sec:CGRF}

As an ensemble method, RF inherits the computational challenges of DT, since building each tree remains expensive when $N$ is large, even though each split only considers a random subset of features (e.g., $\sqrt{P}$ for classification). To further accelerate training, the DICS algorithm can be incorporated by replacing the standard decision tree construction with the Cluster-Guided Classification Tree (CGCT) for each tree in the ensemble. 
We note that the split-candidate dictionary $\Scal$ obtained from bootstrap samples typically varies only slightly, as the perturbation induced by resampling has limited impact on the underlying $K$-means clustering. Therefore, it is unnecessary to recompute $\Scal$ for every CGCT, which helps reduce redundant computations.
We refer to this method as the Cluster-Guided Random Forest (CGRF), which is summarized in \Cref{alg:cgrf}.

\begin{algorithm}[h]
\caption{Cluster-Guided Random Forest (CGRF)}
\label{alg:cgrf}
\begin{algorithmic}[1]

\REQUIRE Feature matrix $\Xbf \in \mathbb{R}^{P \times N}$, labels $\ybm \in [K]^N$, number of trees $B$, number of clusters $m$, maximum depth $D$, number of sampled candidates $k$, feature subset size $\phi_p=\lfloor\sqrt{P}\rfloor$
\ENSURE Forest $\mathcal{F} = \{T_b\}_{b=1}^B$
\vspace{0.5em}
\STATE \textbf{Initialize:} $\Scal \gets \mathrm{DICS}(\Xbf, m)$

\FOR{$b = 1$ to $B$}
    \STATE Draw a bootstrap sample $(\Xbf^{(b)}, \ybm^{(b)})$ of size $N$
    \STATE Sample a feature subset $\mathcal{J}_b \subset [P]$ with $|\mathcal{J}_b| = \phi_p$
    \STATE Construct restricted dictionary $\Scal_b \gets \{T_f \in \Scal : f \in \mathcal{J}_b\}$
    \STATE Grow tree $T_b \gets \mathrm{CGCT}(\Xbf^{(b)}, \ybm^{(b)}, \Scal_b, D, k)$
\ENDFOR

\STATE \textbf{return} $\mathcal{F} = \{T_b\}_{b=1}^B$

\end{algorithmic}
\end{algorithm}

\begin{algorithm}[h]
\caption{\textbf{Fast} \textbf{C}lassification \textbf{G}radient \textbf{B}oosting \textbf{M}achine (FastC-GBM)}
\label{alg:cgxgb_goss}
\begin{algorithmic}[1]
\REQUIRE Feature matrix $\mathbf{X}\in\mathbb{R}^{P\times N}$, labels $y\in[K]^N$, split-candidate dictionary $\mathcal{S}$ (Alg.~\ref{alg:split_candidates}), number of boosting rounds $T$, maximum depth $D$, number of sampled candidates $k$, learning rate $\eta$, GOSS large-gradient size $N_a=\lfloor a\cdot N\rfloor$, GOSS small-gradient sampling size $N_a=\lfloor b\cdot N\rfloor$, column (feature) subsampling size $N_f=\lfloor c\cdot P\rfloor$

\ENSURE Additive multiclass boosted model $F_T(\xbm)$.
\vspace{0.5em}
\STATE \textbf{Initialize:} $F_{0,c}(\bm{x}_i)=\log(\hat{\pi}_c), \forall c\in[K]$, where $\hat{\pi}_c$ is the empirical class proportion.

\FOR{$t=1$ to $T$}

\STATE \textbf{(GOSS step)}
\STATE Update $g_{ic}\gets p_{ic}^{(t-1)}-\Ibb(y_i=c),\quad\forall i\in[N],\forall c\in[K]$.
\STATE Update $h_{ic}\gets p_{ic}^{(t-1)}(1-p_{ic}^{(t-1)}),\quad\forall i\in[N],\forall c\in[K]$.
\STATE $\Acal\gets \{i\in[N]\mid i\in \mbox{Top-}(N_a)\{\|g_1\|_2,\ldots,\|g_N\|_2\}\}$.
\STATE Sample a subset $\Bcal\subset [N]\backslash \Acal$ with $|\Bcal|=N_b$. 

\STATE Update $g_{ic}\gets \frac{1-a}{b}g_{ic},\quad \forall i\in\Bcal,\forall c\in[K]$
\STATE Update $h_{ic}\gets \frac{1-a}{b}h_{ic},\quad \forall i\in\Bcal,\forall c\in[K]$

\STATE \textbf{(Split selection)}
\STATE Sample a subset $\Fcal_{\mathrm{sub}}\subset[P]$ with $|\Fcal_{\mathrm{sub}}|=N_f$
\STATE Sample a subset $\mathcal{C} \subset \mathcal{S}'=\{(f,\tau)\in\Scal:f\in\Fcal_{\mathrm{sub}}\}$ with $|\mathcal{C}| = k$ (without replacement)
\STATE $(f^*, \tau^*) \gets \arg\min_{(f,\tau)\in\mathcal{C}} Q(f,\tau)$ where $Q(f,\tau)$ is evaluated over $\Acal\cup\Bcal$ using \eqref{eq:boosting_quality}.

\STATE $\Scal\gets\Scal/\{(f^*,\tau^*)\}$
\STATE Grow one leaf-wise tree $f_{t}$ from the split $(f^*,\tau^*)$.
\STATE Update model $F_{t}(\xbm_i)\gets F_{t-1}(\xbm_i)+\eta f_{t}(\xbm_i)$.

\ENDFOR

\RETURN $F_T(\xbm)$.
\end{algorithmic}
\end{algorithm}

\subsubsection{Fast classification GBM (FastC-GBM)}\label{sec:FastC-GBM}

For classification tasks, let $f_{t,c}(\xbm)=\sum_{j=1}^Mw^t_{j,c}\Ibb(\xbm\in R_{j}^t),\forall c=1,\ldots,K,$ denote the tree model at iteration $t$ for each class $c$; the resulting class scores are then transformed into probabilities via the softmax function. Each model $f_{t}=[f_{t,1},\ldots,f_{t,c},\ldots,f_{t,K}]^\top\in\Rbb^K$ outputs a $K$-dimensional vector. 
Gradient tree boosting system then proceeds by iteratively solving, at iteration $t$, the following optimization problem based on the current prediction $\hat{\ybm}_i^{(t-1)}$:
\begin{align}\label{eq:boosting_obj}
    \Lcal^{(t)}& =\sum_{i=1}^N\ell(\ybm_i,\hat{\ybm}_i^{(t-1)}+f_t(\xbm_i))+\Omega(f_t),
\end{align}
where the loss function is the multiclass cross-entropy 
\begin{align*}
    \ell(\ybm_i,\hat{\ybm}_i)=-\sum_{c=1}^Ky_{ic}\log p_{ic}, \quad  p_{ic}=\frac{\exp(\hat{y}_{ic})}{\sum_{l=1}^K\exp(\hat{y}_{il})}
\end{align*}
and the regularization term is given by $\Omega(f_t)=\gamma M+\frac{1}{2}\lambda \sum_{c=1}^K\|w^t_{:,c}\|^2$. Taking a second-order approximation of \eqref{eq:boosting_obj} and expanding the regularization term yields
\begin{align}\label{eq:boosting_prob}
    \tilde\Lcal^{(t)} & = \sum_{c=1}^K\sum_{j=1}^M[(\sum_{i\in\Ical_j}g_{ic})w_{j,c}+\frac{1}{2}(\lambda+\sum_{i\in\Ical_j}h_{ic})w_{j,c}^2]+\gamma M,
\end{align}
where $\Ical_j=\{i:\xbm_i\in R^t_j\}$ denote the instance set of leaf $j$ and $g_{ic},h_{ic}$ are the first and second order gradient statistics on the loss function for each class $c$, whose formulas are given by
\begin{align*}
    g_{ic}=p_{ic}-y_{ic}, \quad h_{ic}=p_{ic}(1-p_{ic}), \quad \forall c=1,\ldots,K.
\end{align*}
The optimal leaf weights for each region $R_j$ and each class $c$, obtained by minimizing \eqref{eq:boosting_prob}, are
\begin{align}\label{eq:boosting_weights}
    w_{j,c}^*=-\frac{\sum_{i\in\Ical_j}g_{ic}}{\lambda+\sum_{i\in\Ical_j}h_{ic}}, \quad \forall c=1,\ldots,K.
\end{align}
Let $\mathcal{I}_L$ and $\mathcal{I}_R$ denote the instance sets of the left and right child nodes induced by a candidate split $(f,\tau)$, and let $\mathcal{I}=\mathcal{I}_L \cup \mathcal{I}_R$. Then the loss reduction is given by
\begin{align}\label{eq:boosting_quality}
    Q(f,\tau)=\frac{1}{2}\sum_{c=1}^K\left[ \frac{(\sum_{i\in\Ical_L}g_{ic})^2}{\lambda+\sum_{i\in\Ical_L}h_{ic}}+\frac{(\sum_{i\in\Ical_R}g_{ic})^2}{\lambda+\sum_{i\in\Ical_R}h_{ic}}-\frac{(\sum_{i\in\Ical}g_{ic})^2}{\lambda+\sum_{i\in\Ical}h_{ic}}\right]-\gamma .
\end{align}
which can be used as evaluating the quality of the candidate split $Q(f,\tau)$.

By using DICS and incorporating the Gradient-based One-Side Sampling (GOSS) technique~\cite{ke2017lightgbm} and column (feature) subsampling technique~\cite{chen2016xgboost,ke2017lightgbm}, our \textbf{Fast} \textbf{C}lassification \textbf{G}radient \textbf{B}oosting \textbf{M}achine (\textit{FastC-GBM}) is summarized in \Cref{alg:cgxgb_goss}. 

\subsection{Technical Proofs}\label{sec:proof}

\begin{proof}[Proof of Theorem~\ref{thm:main}]
This proof consists of two parts. In Part (i), we show that under the density regularity condition, $\delta=d_H(\Pcal_{f,\bar\tau_*^{(f)}},\Pcal_{f,\tilde\tau_*^{(f)}})=\Ocal_P(\sqrt{N})$. In Part (ii), we establish that the difference between the two optimal gain values is bounded.

\textit{\textbf{Part (i).}} Fix a feature $f$, and for notational simplicity write $\tau := \tau^{(f)}$.
We first establish the convergence rate of the empirical centroids $\hat{\mu}_c$ to the population means $\mu_c$. Assume the feature $f$ is bounded in $[a,b]$. By Hoeffding's inequality, for each class $c$ and any $\varepsilon > 0$,
\[
\mathbb{P}\big(|\hat{\mu}_c - \mu_c| \ge \varepsilon\big)
\le 2 \exp\left(-\frac{2 n_c \varepsilon^2}{(b-a)^2}\right),
\]
where $n_c$ is the number of samples in class $c$. Since $n_c \asymp N$, it follows that with probability at least $1-\gamma$,
\[
|\hat{\mu}_c - \mu_c|
\le (b-a)\sqrt{\frac{\log(2/\gamma)}{2 n_c}}
= \mathcal{O}\!\left(\frac{1}{\sqrt{N}}\right).
\]

In DICS, the candidate split is defined as $\tilde{\tau} = \zeta(\hat{\mu}_{c_1}, \hat{\mu}_{c_2})
= \frac{w_1 \hat{\mu}_{c_1} + w_2 \hat{\mu}_{c_2}}{w_1 + w_2}$,
where $w_1 = \hat{\sigma}_{c_2}$ and $w_2 = \hat{\sigma}_{c_1}$. The mapping $\zeta$ is linear and thus Lipschitz continuous with constant $L = \|\nabla \zeta\|_2 = \frac{\sqrt{w_1^2 + w_2^2}}{w_1 + w_2} \le 1$. Hence, $|\tilde{\tau} - \tau^*|
\le L \sum_c |\hat{\mu}_c - \mu_c|
= \mathcal{O}_P\!\left(\frac{1}{\sqrt{N}}\right)$.

Let $\tau^*$ and $\bar{\tau}^*$ denote the population and empirical minimizers of the Gini impurity function $Q(\tau)$ and $Q_N(\tau)$, respectively. For a fixed feature $f$, the split is determined by the indicator $\mathbb{I}\{X^{(f)} < \tau\}$, which defines a step function with a discontinuity at $\tau$. Hence, the empirical objective is piecewise constant and its maximizer lies on a grid of order statistics of size $1/N$. Standard results from change-point estimation imply $|\tau^* - \bar{\tau}^*| = \mathcal{O}_P\!\left(\frac{1}{N}\right)$. 

Therefore, $|\tilde{\tau} - \bar{\tau}^*|
\le |\tilde{\tau} - \tau^*| + |\tau^* - \bar{\tau}^*|
= \mathcal{O}_P\!\left(\frac{1}{\sqrt{N}}\right)$, which implies that $|\tilde{\tau}^* - \bar{\tau}^*|=\mathcal{O}_P\!\left(\frac{1}{\sqrt{N}}\right)$.  
Define the Hamming distance
\[
\delta = \sum_{i=1}^N \mathbb{I}\left\{X_i^{(f)} \in
\left[\min(\bar{\tau}^*, \tilde{\tau}^*), \max(\bar{\tau}^*, \tilde{\tau}^*)\right]\right\}.
\]
Assume the marginal density $p(x^{(f)})$ is bounded by $M$ in a neighborhood of $\tau^*$. Then
\[
\mathbb{E}[\delta]
= N \int_{\min(\bar{\tau}^*, \tilde{\tau}^*)}^{\max(\bar{\tau}^*, \tilde{\tau}^*)}
p(x^{(f)})\,dx
\le N M |\tilde{\tau}^* - \bar{\tau}^*|.
\]
Using $|\tilde{\tau}^* - \bar{\tau}^*| = \mathcal{O}_P(N^{-1/2})$, we obtain $\mathbb{E}[\delta] = \mathcal{O}(\sqrt{N})$. By Markov's inequality, for any $\varepsilon > 0$, $\mathbb{P}(\delta > \varepsilon \sqrt{N})
\le \frac{\mathbb{E}[\delta]}{\varepsilon \sqrt{N}} \to 0,$
as $N \to \infty$. Hence, $\delta = \mathcal{O}_P(\sqrt{N})$, which completes the proof.

\textbf{\textit{Part (ii). }}
    For notational simplicity, we write $\mathcal{P} = (\mathcal{I}_L, \mathcal{I}_R)$ to denote the partition $\mathcal{P}_{f,\bar{\tau}_*^{(f)}}$, and $\mathcal{P}' = (\mathcal{I}'_L, \mathcal{I}'_R)$ to denote the partition $\mathcal{P}_{f,\tilde{\tau}_*^{(f)}}$. Then $\Pcal$ has sizes $n_L = |\mathcal{I}_L|$ and $n_R = |\mathcal{I}_R|$. The class counts within each partition are defined as $n_{L,c} = |\{ i \in \mathcal{I}_L : y_i = c \}|$ and $n_{R,c} = |\{ i \in \mathcal{I}_R : y_i = c \}|$ for each class $c \in [K]$. 

    If $d_H(\mathcal{P}, \mathcal{P}') = \delta$, then there are $\delta$ samples are reassigned, either from $\mathcal{I}_L$ to $\mathcal{I}_R$ or vice versa. More precisely, let $\delta_1$ denote the number of samples moving from $\mathcal{I}_L$ to $\mathcal{I}_R$ (referred to as \textit{left-to-right}), and $\delta_2$ the number moving from $\mathcal{I}_R$ to $\mathcal{I}_L$ (referred to as \textit{right-to-left}), so that $\delta_1 + \delta_2 = \delta$. Accordingly, the subset sizes become $n_L' = n_L - \delta_1 + \delta_2$ and $n_R' = n_R + \delta_1 - \delta_2$.

    We next examine the class counts in the left and right subsets. For each left-to-right move, there exists a class $c$ such that $n_{L,c} \to n_{L,c} - 1$ and $n_{R,c} \to n_{R,c} + 1$. Similarly, for each right-to-left move, there exists a class $c$ such that $n_{L,c} \to n_{L,c} + 1$ and $n_{R,c} \to n_{R,c} - 1$.
    
    For a set $\mathcal{I}$ with $|\mathcal{I}| = n$, we consider the Gini index in the form $G(\mathcal{I}) = 1 - \sum_{c=1}^K \frac{n_c^2}{n^2}$, where $n_c$ denotes the number of samples in $\mathcal{I}$ that belong to class $c$. Then the objective can be rewritten as 
    \begin{align*}
        Q(\Pcal) = \frac{n_L}{N}G(\Ical_L)+\frac{n_R}{N}G(\Ical_R)=1-\frac{1}{N}\left( \frac{\sum_{c=1}^Kn_{L,c}^2}{n_L}+\frac{\sum_{c=1}^Kn_{R,c}^2}{n_R} \right)
    \end{align*}
    Hence the difference of two gains of $\Pcal$ and $\Pcal'$ can be found by
    \begin{align}\label{eq:diff_Q}
        Q(\Pcal')-Q(\Pcal) = \frac{1}{N}\left( \frac{\sum_{c=1}^K(n'_{L,c})^2}{n'_L}- \frac{\sum_{c=1}^Kn_{L,c}^2}{n_L}+\frac{\sum_{c=1}^K(n'_{R,c})^2}{n'_R}-\frac{\sum_{c=1}^Kn_{R,c}^2}{n_R} \right)
    \end{align}
    Let $S_L = \sum_{c=1}^K n_{L,c}^2$ and $S_R = \sum_{c=1}^K n_{R,c}^2$, and define $S_L', S_R'$ analogously. Denote: $I_1=\frac{S_L'}{n'_L}-\frac{S_L}{n_L}, I_2=\frac{S_R'}{n'_R}-\frac{S_R}{n_R}$.
    We first bound $I_1$. Write
    \begin{align*}
        I_1=\frac{S_L'}{n'_L}-\frac{S_L}{n_L} = \frac{S_L'-S_L}{n_L}+S_L'\left( \frac{1}{n_L'}-\frac{1}{n_L} \right).
    \end{align*}

    Each moved sample changes exactly one class count by $\pm1$, so that $(n_{L,c}')^2-n_{L,c}^2=\pm 2n_{L,c}+1$, which contributes at most $|(n_{L,c}')^2-n_{L,c}^2|\leq 2n_{L}+1<3n_L$. So for $\delta$ samples are moved, so that 
    \begin{align}\label{eq:s_l_upper}
        |S_L'-S_L|=\left|\sum_{c=1}^K(n_{L,c}')^2-\sum_{c=1}^Kn_{L,c}^2 \right|\leq \sum_{c=1}^K|(n_{L,c}')^2-n_{L,c}^2|<3\delta n_L 
    \end{align}
    Moreover, one can easily show that $S_L' \le n_L'^2$. Combining the inequalities $S_L' \le n_L'^2$ and $|n_L'-n_L|\le \delta$ into \eqref{eq:s_l_upper} yields
    \begin{align*}
        |I_1| \le \frac{|S_L'-S_L|}{n_L}+\frac{|n_L'-n_L|S_L'}{n_Ln'_L} <3\delta+\delta\frac{n'_L}{n_L} = (4+\frac{\delta}{n_L})\delta.
    \end{align*}
    Similarly, we can show that $|I_2|\leq (4+\frac{\delta}{n_R})\delta$. Substituting the inequalities of $I_1$ and $I_2$ into \eqref{eq:diff_Q} gives
    \begin{align}\label{eq:diff_Q_ineq}
        |Q(\Pcal)-Q(\Pcal')|< (8+\frac{\delta}{n_L}+\frac{\delta}{n_R})\frac{\delta}{N}\leq (8+\frac{1}{\alpha(1-\alpha)}\cdot\frac{\delta}{N})\frac{\delta}{N},
    \end{align}
    where $\alpha=\min_f\{\frac{n_L^{(f)}}{N},\frac{n_R^{(f)}}{N}\}$.
    Finally, since the above bound holds uniformly over $f$, we have
    \begin{align*}
        \bar Q^* =\min_f Q(f,\bar \tau_*^{(f)}) &\leq \min_f Q(f,\tilde \tau_*^{(f)})+(8+\frac{1}{\alpha(1-\alpha)}\cdot\frac{\delta}{N})\frac{\delta}{N} \\
        & = \tilde Q^*+(8+\frac{1}{\alpha(1-\alpha)}\cdot\frac{\delta}{N})\frac{\delta}{N}
    \end{align*}
    and similarly,
    \[
    \tilde Q^*\le \bar Q^* + (8+\frac{1}{\alpha(1-\alpha)}\cdot\frac{\delta}{N})\frac{\delta}{N}.
    \]
    Here we use the fact that $\bar Q^*=\min_f Q(f,\bar \tau_*^{(f)})=\min_{(f,\tau)\in\Scal_{\mathrm{DT}}}Q(f,\tau)$ and $\tilde Q^*=\min_f Q(f,\tilde \tau_*^{(f)})=\min_{(f,\tau)\in\Scal_{\mathrm{DICS}}}Q(f,\tau)$.
    Combining the two inequalities above completes the proof.

\end{proof}

\subsection{Parameter sensitivity}\label{sec:app_sensitivity}

\begin{figure}[ht]
    \centering

    \begin{subfigure}[t]{0.48\textwidth}
        \centering
        \includegraphics[width=\linewidth]{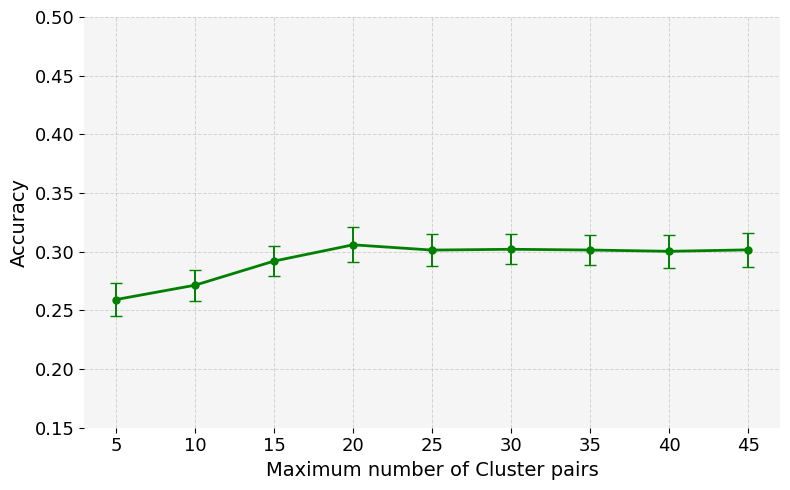}
        \caption{Test accuracy vs.\ maximum number of cluster pairs ($m$) on CIFAR-10.}
        \label{fig:sensitivity_m}
    \end{subfigure}
    \hfill
    \begin{subfigure}[t]{0.48\textwidth}
        \centering
        \includegraphics[width=\linewidth]{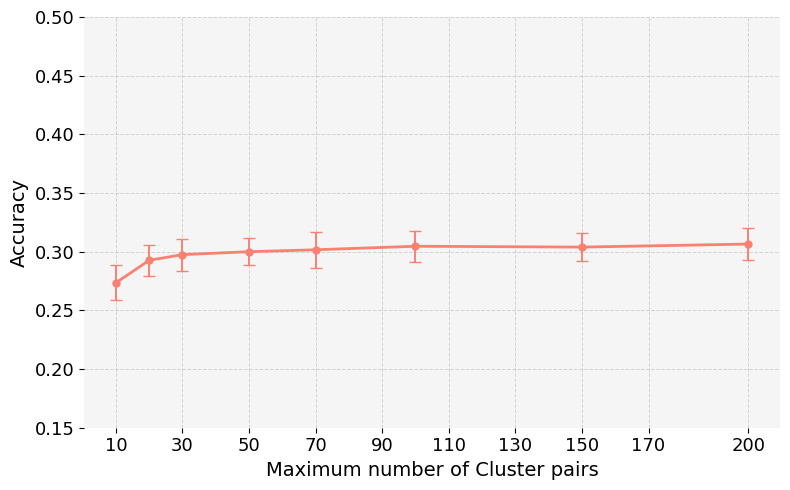}
        \caption{Test accuracy vs.\ maximum number of feature-threshold pairs ($k$) on CIFAR-10.}
        \label{fig:sensitivity_k}
    \end{subfigure}

    \vspace{0.4cm}

    \begin{subfigure}[t]{0.48\textwidth}
        \centering
        \includegraphics[width=\linewidth]{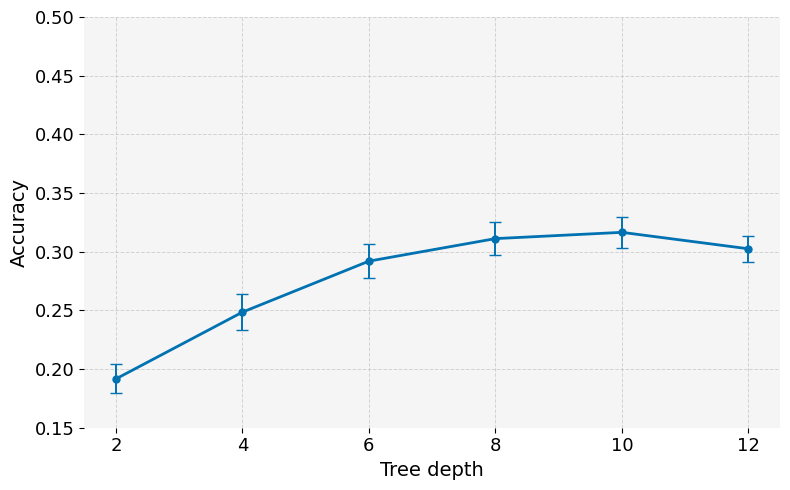}
        \caption{Test accuracy vs.\ maximum tree depth ($D$) on CIFAR-10.}
        \label{fig:sensitivity_depth}
    \end{subfigure}
    \hfill
    \begin{subfigure}[t]{0.48\textwidth}
        \centering
        \includegraphics[width=\linewidth]{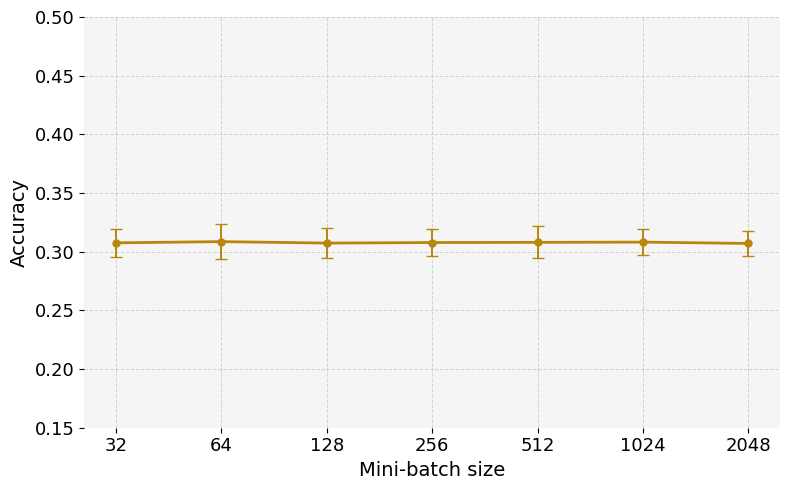}
        \caption{Test accuracy vs.\ mini-batch size ($b$) on CIFAR-10.}
        \label{fig:sensitivity_batch}
    \end{subfigure}

    \caption{
    Parameter sensitivity of CGCT on CIFAR-10 with respect to
    (a) the maximum number of cluster pairs ($m$),
    (b) the maximum number of feature-threshold pairs ($k$),
    (c) the maximum tree depth ($D$), and
    (d) the mini-batch size ($b$).
    Each point represents the mean test accuracy over 50 runs, and error bars denote $\pm 3$ standard deviations.
    }
    \label{fig:parameter_sensitivity}
\end{figure}

We analyze the sensitivity of CGCT to four key parameters: the maximum number of cluster pairs $m$ from \Cref{alg:split_candidates}, the length of the feature-threshold set $k$ where $|\mathcal{C}|=k$, the maximum tree depth $D$, and the mini-batch size $b$ used in mini-batch K-means. These parameters control the number of candidate splits considered during tree construction, the complexity of the resulting tree, and the computational cost of candidate generation.

\Cref{fig:parameter_sensitivity}(a) shows that increasing $m$ improves accuracy at first, but the performance quickly stabilizes after a moderate value. This suggests that a small number of cluster pairs is sufficient to capture useful split candidates, while additional pairs provide limited benefit.

\Cref{fig:parameter_sensitivity}(b) illustrates a similar trend for $k$. Initially, accuracy increases as more feature-threshold pairs are considered. However, beyond a specific threshold, the improvement diminishes significantly, and the curve gradually flattens. This suggests that once a sufficient number of informative candidates is included, adding more feature-threshold pairs yields only marginal benefit.

\Cref{fig:parameter_sensitivity}(c) shows the effect of the maximum tree depth $D$. Accuracy improves substantially as the depth increases from shallow trees and gradually stabilizes at larger depths. The highest accuracy is obtained around $D=8$–$10$, while increasing the depth further provides no additional improvement and results in a slight decrease at $D=12$. This indicates that a moderate tree depth is sufficient to capture the useful structure in the candidate splits without requiring unnecessarily deep trees.

Finally, \Cref{fig:parameter_sensitivity}(d) shows that CGCT is highly stable with respect to the mini-batch size $b$. Across a wide range of mini-batch sizes, from $32$ to $2048$, the test accuracy remains nearly unchanged. This suggests that the candidate splits generated by DICS are not highly sensitive to the mini-batch size used in K-means, and relatively small mini-batches are sufficient to obtain stable predictive performance.

Overall, these results show that CGCT is not highly sensitive to large values of the aforementioned parameters. Moderate parameter choices are sufficient to
obtain stable accuracy while maintaining efficient split generation and tree construction.

\subsection{Empirical Validation of DICS Split Quality}
\label{sec:app_split_quality}

We further evaluate the split quality of DICS from both tree-level and
split-level perspectives. First, we independently grow exhaustive DT and
CGCT on CIFAR-10 with maximum depth $D=8$, allowing each method to select
its own feature--threshold pair throughout tree construction. Table~\ref{tab:cifar10_depth_gain}
reports the mean local gain, sample-weighted impurity reduction, and
cumulative gain across depth. Overall, CGCT closely follows the impurity
reduction achieved by exhaustive DT despite the two trees being grown
independently. At the final splitting depth, CGCT achieves a cumulative
gain of $0.123$, compared with $0.127$ for exhaustive DT, corresponding
to $97.18\%$ of the cumulative gain achieved by exhaustive split search.

We additionally examine the split-level behavior in Remark~1 using the
synthetic-data setting by varying the sample size from $N=1{,}000$ to
$N=20{,}000$. For each sample size, DICS and exhaustive DT independently
select their best root-level feature--threshold pair, and the results are
averaged over 40 independent runs. As shown in
Table~\ref{tab:split_scaling}, the mean gain gap $\Delta_G$ decreases from
$0.0039$ to $0.0011$ as $N$ increases, while
$\sqrt{N}\Delta_G$ remains of comparable magnitude across the evaluated
sample sizes. This behavior is consistent with the
$O_P(N^{-1/2})$ split-level rate in Remark~1 and provides additional
empirical evidence that the optimal DICS gain approaches that of exhaustive
DT as the sample size increases.

\begin{table*}[http]
\centering

\setlength{\tabcolsep}{6pt}
\renewcommand{\arraystretch}{1.08}

\begin{tabular}{c
                cc
                cc
                cc
                c}
\toprule

& \multicolumn{2}{c}{Mean Local Gain $\bar{g}_d$}
& \multicolumn{2}{c}{Weighted Gain $G_d^{(w)}$}
& \multicolumn{2}{c}{Cumulative Gain $C_d$}
& $R_d$ (\%) \\

\cmidrule(lr){2-3}
\cmidrule(lr){4-5}
\cmidrule(lr){6-7}

Depth
& DT & CGCT
& DT & CGCT
& DT & CGCT
& \\

\midrule

1
& 0.019 & 0.019
& 0.019 & 0.019
& 0.019 & 0.019
& 100.00 \\

2
& 0.013 & 0.012
& 0.012 & 0.012
& 0.031 & 0.031
& 99.46 \\

3
& 0.015 & 0.014
& 0.013 & 0.013
& 0.044 & 0.043
& 98.04 \\

4
& 0.015 & 0.015
& 0.013 & 0.013
& 0.057 & 0.057
& 99.37 \\

5
& 0.018 & 0.017
& 0.013 & 0.013
& 0.070 & 0.070
& 99.40 \\

6
& 0.023 & 0.021
& 0.015 & 0.014
& 0.085 & 0.084
& 98.58 \\

7
& 0.037 & 0.030
& 0.018 & 0.017
& 0.103 & 0.101
& 97.71 \\

8
& 0.061 & 0.051
& 0.023 & 0.022
& \textbf{0.127} & \textbf{0.123}
& \textbf{97.18} \\

\bottomrule
\end{tabular}

\caption{
Depth-wise impurity reduction of independently grown exhaustive DT and CGCT
on CIFAR-10 with maximum depth $D=8$, where Depth~1 denotes the root level.
For depth $d$, $\bar{g}_d =
|\mathcal{I}_d|^{-1}\sum_{j\in\mathcal{I}_d} g_j$ is the mean local Gini gain,
$G_d^{(w)} =
\sum_{j\in\mathcal{I}_d}(n_j/N)g_j$ is the sample-weighted impurity reduction,
$C_d = \sum_{r=1}^{d} G_r^{(w)}$ is the cumulative weighted gain, and
$R_d = C_d^{\mathrm{CGCT}}/C_d^{\mathrm{DT}}$ is the cumulative gain
retained by CGCT relative to exhaustive DT.
}
\label{tab:cifar10_depth_gain}

\end{table*}

\begin{table}[http]
\centering

\setlength{\tabcolsep}{10pt}
\renewcommand{\arraystretch}{1.08}

\begin{tabular}{rcc}
\toprule
$N$
& Gain Gap $\Delta_G$
& $\sqrt{N}\,\Delta_G$ \\
\midrule

1{,}000
& 0.0039
& 0.1227 \\

2{,}000
& 0.0025
& 0.1100 \\

5{,}000
& 0.0017
& 0.1183 \\

10{,}000
& 0.0012
& 0.1242 \\

20{,}000
& 0.0011
& 0.1492 \\

\bottomrule
\end{tabular}

\caption{
Split-level gain gap between DICS and exhaustive DT across increasing
sample sizes. Results are averaged over 40 independent runs.
}
\label{tab:split_scaling}

\end{table}

\subsection{Additional Experiments}\label{sec:app_sim}

This section presents additional experiments on synthetic data with $\rho = 0$, where the generated features are independent of each other, as well as descriptions of the real datasets used in this paper. Additional comparisons of various algorithms are reported in \Cref{tab:synthetic_data_CGCT_rho0}--\ref{tab:xgb_lgb_cgxgb_rho0}. A summary of the real datasets is provided in \Cref{tab:summary_real}.

\begin{table*}[http]
\centering
\footnotesize
\setlength{\tabcolsep}{1.7pt}
\begin{tabular}{cc cccc cccc}
\toprule
\multirow{2}{*}{$(N,P)$} & \multirow{2}{*}{Class}
& \multicolumn{3}{c}{Time (sec)$\downarrow$}
& \multicolumn{1}{c}{Speedup$\uparrow$}
& \multicolumn{3}{c}{ACC$\uparrow$} \\
\cmidrule(lr){3-5} \cmidrule(lr){6-6} \cmidrule(lr){7-9}
& & DT & BDTKS & CGCT
& $t_{\mathrm{DT}}/t_{\mathrm{CGCT}}$
& DT & BDTKS & CGCT \\
\midrule

\multirow{3}{*}{10000,1000}
& 3  & 2.59 (0.02)   & 31.84   & \best{0.18 (0.02)}
     & 14.39
     & \best{0.67 (0.00)} & 0.53   & 0.65 (0.00) \\
& 5  & 2.49 (0.00)   & 35.54 & \best{0.27 (0.02)}
     & 9.22
     & \best{0.50 (0.00)} & 0.31   & 0.49 (0.00) \\
& 10 & 2.64 (0.00) & 38.21 & \best{0.30 (0.02)}
     & 8.80
     & \best{0.46 (0.00)} & 0.30 & 0.44 (0.01) \\
\cmidrule(lr){1-9}

\multirow{3}{*}{5000,7000}
& 3  & 8.15 (0.09)   & 29.65  & \best{0.63 (0.03)}
     & 12.94
     & \best{0.69 (0.00)} & 0.62  & 0.68 (0.01) \\
& 5  & 8.56 (0.04)   & 29.83   & \best{0.92 (0.05)}
     & 9.30
     & \best{0.39 (0.00)} & 0.30  & 0.37 (0.01) \\
& 10 & 8.67 (0.02)  & 33.40  & \best{1.20 (0.05)}
     & 7.22
     & \best{0.34 (0.00)} & 0.24 & 0.31 (0.01) \\
\cmidrule(lr){1-9}

\multirow{3}{*}{20000,5000}
& 3  & 26.50 (0.33)    & 118.76 & \best{1.20 (0.10)}
     & 22.08
     & \best{0.47 (0.00)} & 0.45 & \best{0.47 (0.00)} \\
& 5  & 29.46 (0.04) & 125.70 & \best{1.69 (0.01)}
     & 17.43
     & \best{0.53 (0.00)} & 0.41 & 0.51 (0.00) \\
& 10 & 29.53 (0.04)  & 131.71  & \best{2.07 (0.07)}
     & 14.26
     & \best{0.50 (0.00)} & 0.34 & 0.48 (0.00) \\
\bottomrule
\end{tabular}
\caption{Accuracy and training time (in seconds) for DT, BDTKS, and CGCT in real datasets. The ratio $t_{\mathrm{DT}}/t_{\mathrm{CGCT}}$ quantifies the training time improvement of CGCT relative to DT. Here $\rho=0$.}
\label{tab:synthetic_data_CGCT_rho0}
\end{table*}

\begin{table*}[http]
\centering
\footnotesize
\begin{tabular}{cc cc c cc}
\toprule
\multirow{2}{*}{$(N,P)$} & \multirow{2}{*}{Class}
& \multicolumn{2}{c}{Time (sec)$\downarrow$}
& \multicolumn{1}{c}{Speedup$\uparrow$}
& \multicolumn{2}{c}{ACC$\uparrow$} \\
\cmidrule(lr){3-4} \cmidrule(lr){5-5} \cmidrule(lr){6-7}
& & RF & CGRF & $t_{\mathrm{RF}}/t_{\mathrm{CGRF}}$ & RF & CGRF \\
\midrule

\multirow{3}{*}{10000,1000}
& 3  & 4.16 (0.04)   & \best{0.34 (0.02)}  & 12.23  & \best{0.70 (0.00)}  & \best{0.70 (0.00)} \\
& 5  & 4.13 (0.05)   & \best{0.39 (0.00)}  & 10.59  & \best{0.54 (0.00)} & \best{0.54 (0.01)} \\
& 10 & 4.54 (0.00)   & \best{0.50 (0.02)}  & 9.08  & \best{0.47 (0.00)} & \best{0.47 (0.00)} \\
\cmidrule(lr){1-7}

\multirow{3}{*}{5000,7000}
& 3  & 13.67 (0.46)  & \best{0.86 (0.07)} & 15.89 & \best{0.73 (0.00)}  & \best{0.73 (0.01)} \\
& 5  & 14.13 (0.15)    & \best{1.22 (0.03)} & 11.58 & \best{0.44 (0.01)} & \best{0.44 (0.00)} \\
& 10 & 14.51 (0.16)   & \best{1.71 (0.07)}  & 8.48 & \best{0.36 (0.01)} & \best{0.36 (0.01)} \\
\cmidrule(lr){1-7}

\multirow{3}{*}{20000,5000}
& 3  & 45.35 (1.86) & \best{1.52 (0.05)}  & 29.83 & \best{0.51 (0.00)} & 0.50 (0.01) \\
& 5  & 57.20 (1.31)  & \best{1.71 (0.08)} & 33.45 & \best{0.55 (0.00)}  & \best{0.55 (0.00)} \\
& 10 & 54.83 (0.24)& \best{2.09 (0.04)} & 26.23 & \best{0.51 (0.00)}  & \best{0.51 (0.00)} \\
\bottomrule
\end{tabular}
\caption{Accuracy and training time (in seconds) for RF and CGRF in synthetic data. The ratio $t_{\mathrm{RF}}/t_{\mathrm{CGRF}}$ quantifies the training time improvement of CGRF relative to RF. Here $\rho=0$.}
\label{tab:rf_cgrf_synth_speedup_rho0}
\end{table*}

\begin{table*}[http]
\centering
\setlength{\tabcolsep}{4.0pt}
\resizebox{\linewidth}{!}{
\begin{tabular}{cc ccc cc ccc}
\toprule
\multirow{2}{*}{$(N,P)$} & \multirow{2}{*}{Class}
& \multicolumn{3}{c}{Time (sec)$\downarrow$}
& \multicolumn{2}{c}{Speedup$\uparrow$}
& \multicolumn{3}{c}{ACC$\uparrow$} \\
\cmidrule(lr){3-5} \cmidrule(lr){6-7} \cmidrule(lr){8-10}
& & XGBoost & LightGBM & CGXGB
& $\rho_1$
& $\rho_1$
& XGBoost & LightGBM & CGXGB \\
\midrule

\multirow{3}{*}{10000,1000}
& 3  & 1.43 (0.01)  & 0.97 (0.00) & \best{0.30 (0.01)} & 4.77  & 3.23  & \best{0.71 (0.01)} & \best{0.71 (0.01)} & 0.70 (0.01) \\
& 5  & 2.36 (0.00)  & 1.59 (0.00) & \best{0.37 (0.00)} & 6.38 & 4.30  & 0.54 (0.00) & \best{0.55 (0.01)} & \best{0.55 (0.01)} \\
& 10 & 4.50 (0.01) & 2.93 (0.01) & \best{0.44 (0.01)} & 10.23 & 6.66  & \best{0.48 (0.01)} & 0.47 (0.02) & 0.45 (0.01) \\
\cmidrule(lr){1-10}

\multirow{3}{*}{5000,7000}
& 3  & 6.73 (0.05) & 3.13 (0.06) & \best{0.24 (0.01)} & 28.04 & 13.04 & \best{0.75 (0.02)} & \best{0.75 (0.02)} & 0.74 (0.01) \\
& 5  & 10.81 (0.23) & 4.80 (0.15) & \best{0.29 (0.00)} & 37.28 & 16.55 & \best{0.41 (0.02)} & \best{0.41 (0.02)} & 0.38 (0.01) \\
& 10 & 18.16 (0.46) & 9.30 (0.59) & \best{0.37 (0.01)} & 49.08 & 25.13 & \best{0.35 (0.01)} & 0.34 (0.01) & 0.33 (0.01) \\
\cmidrule(lr){1-10}

\multirow{3}{*}{20000,5000}
& 3  & 9.40 (0.24)  & 5.73 (0.17) & \best{0.68 (0.00)} & 13.82 & 8.43  & 0.54 (0.01) & \best{0.55 (0.01)} & 0.52 (0.01) \\
& 5  & 15.06 (0.19)  & 8.13 (0.08) & \best{0.77 (0.00)} & 19.56 & 10.56 & \best{0.55 (0.01)} & \best{0.55 (0.01)} & 0.54 (0.01) \\
& 10 & 31.04 (0.72) & 15.10 (0.35) & \best{1.09 (0.03)} & 28.48 & 13.85 & \best{0.51 (0.01)} & \best{0.51 (0.01)} & 0.50 (0.01) \\
\bottomrule
\end{tabular}
}
\caption{Accuracy and training time (in seconds) for XGBoost, LightGBM and FastC-GBM in synthetic data. The ratios $\rho_1=t_{\mathrm{XGBoost}}/t_{\mathrm{FastC-GBM}}$ and $\rho_2=t_{\mathrm{LightGBM}}/t_{\mathrm{FastC-GBM}}$ quantify the training time improvement of FastC-GBM relative to XGBoost and LightGBM, respectively. Here $\rho=0$.}
\label{tab:xgb_lgb_cgxgb_rho0}
\end{table*}

\begin{table*}[http]
\centering
\rowcolors{2}{gray!6}{white}
\setlength{\tabcolsep}{10pt}
\renewcommand{\arraystretch}{1.2}

\begin{tabular}{lccc}
\toprule
\rowcolor{gray!15}
\textbf{Dataset} & $\mathbf{N}$ & $\mathbf{P}$ & $\mathbf{K}$ \\
\midrule
Helena          & 65{,}196  & 27   & 100 \\
Spambase        & 4{,}601   & 57   & 2   \\
Santander       & 200{,}000 & 200  & 2   \\
CIFAR-10        & 60{,}000  & 3{,}072 & 10  \\
MNIST           & 70{,}000  & 784  & 10  \\
Fashion-MNIST   & 70{,}000  & 784  & 10  \\
\bottomrule
\end{tabular}
\caption{Summary of datasets used in the experiments, including sample size ($N$), number of features ($P$), and number of classes ($K$).}
\label{tab:summary_real}
\end{table*}

\end{document}